\documentclass{article} %
\usepackage{iclr2021_conference,times}
\usepackage{comment}

\usepackage{amsmath,amsfonts,bm}

\def\figref#1{figure~\ref{#1}}

\def\eqref#1{equation~\ref{#1}}
\def\Eqref#1{Equation~\ref{#1}}

\def\1{\bm{1}}

\def\eps{{\epsilon}}

\def\vx{{\bm{x}}}

\DeclareMathAlphabet{\mathsfit}{\encodingdefault}{\sfdefault}{m}{sl}
\SetMathAlphabet{\mathsfit}{bold}{\encodingdefault}{\sfdefault}{bx}{n}

\newcommand{\pdata}{p_{\rm{data}}}

\newcommand{\E}{\mathbb{E}}

\newcommand{\abs}[1]{\lvert#1\rvert}

\newcommand{\bs}[1]{\boldsymbol{#1}}

\newcommand{\be}{\begin{equation}}
	\newcommand{\ee}{\end{equation}}

\definecolor{Gray}{gray}{0.85}
\definecolor{LightCyan}{rgb}{0.88,1,1}

\makeatletter
\usepackage{xspace}
\def\@onedot{\ifx\@let@token.\else.\null\fi\xspace}
\DeclareRobustCommand\onedot{\futurelet\@let@token\@onedot}

\newcommand{\tabref}[1]{Tab\onedot~\ref{#1}}
\newcommand{\thmref}[1]{Theorem~\ref{#1}}

\newcommand{\bfx}{\mathbf{x}}

\newcommand{\bfe}{{\bs{\epsilon}}}

\def\eg{\emph{e.g}\onedot}

\def\ie{\emph{i.e}\onedot}

\newcommand{\chenlin}[1]{{\color{cyan} [CM: {#1}]}}

\usepackage[textsize=tiny]{todonotes}

\usepackage{url}
\usepackage{xcolor}
\usepackage{hyperref}
\hypersetup{colorlinks,breaklinks,
            citecolor=[rgb]{0,0.47, 0.75},
            linkcolor=[rgb]{0.0, 0.5, 0.0},
            urlcolor=green}
            
\title{Improved Autoregressive Modeling \\ with Distribution Smoothing}

\author{Chenlin Meng, Jiaming Song, Yang Song, Shengjia Zhao \& Stefano Ermon\\
Stanford University\\
\texttt{\{chenlin,tsong,yangsong,sjzhao,ermon\}@cs.stanford.edu}\\
}

\iclrfinalcopy %
\begin{document}
\maketitle

\begin{abstract}
While autoregressive models excel at image compression, their sample quality is often lacking. 
Although not  realistic,  generated images often have high likelihood according to the model, resembling the case of adversarial examples.
Inspired by a successful adversarial defense method, we incorporate randomized smoothing into autoregressive generative modeling. 
We first model a smoothed version of the data distribution, and then reverse the smoothing process to recover the original data distribution. 
This  procedure drastically improves the sample quality of existing autoregressive models on several synthetic and real-world image datasets while obtaining competitive likelihoods on synthetic datasets.

\end{abstract}

\section{Introduction}
Autoregressive models have exhibited promising results in a variety of downstream tasks. For instance, they have shown success in compressing images~\citep{minnen2018joint}, synthesizing speech~\citep{oord2016wavenet} and modeling complex decision rules in games~\citep{vinyals2019grandmaster}. 
However, the sample quality of autoregressive models on real-world image datasets is still lacking.

Poor sample quality might be explained by the manifold hypothesis: many real world data distributions (\eg natural images) lie in the vicinity of a low-dimensional manifold~\citep{belkin2003laplacian},  
leading to complicated densities with sharp transitions (\ie high Lipschitz constants), which are known to be difficult to model for density models such as normalizing flows~\citep{cornish2019relaxing}. Since each conditional of an autoregressive model is a 1-dimensional normalizing flow (given a fixed context of previous pixels), a high Lipschitz constant will likely hinder learning of autoregressive models.

Another reason for poor sample quality is the ``compounding error" issue in autoregressive modeling.
To see this, we note that an autoregressive model relies on the previously generated context to make a prediction; once a mistake is made, the model is likely to make another mistake which compounds~\citep{kaariainen2006lower}, eventually resulting in questionable and unrealistic samples. 
Intuitively, one would expect the model to assign low-likelihoods to such unrealistic images, however, this is not always the case.
In fact, the generated samples, although appearing unrealistic, often are assigned  high-likelihoods by the autoregressive model, resembling an ``adversarial example"~\citep{szegedy2013intriguing,biggio2013evasion}, an input that causes the model to output an incorrect answer with high confidence.

Inspired by the recent success of randomized smoothing techniques in adversarial defense~\citep{cohen2019certified}, we propose to apply randomized smoothing to autoregressive generative modeling. More specifically, we propose to address a density estimation problem via a two-stage process. Unlike \cite{cohen2019certified} which applies smoothing to the model to make it more robust, we apply smoothing to the data distribution. Specifically, we convolve a symmetric and stationary noise distribution with the data distribution to obtain a new ``smoother" distribution.  
In the first stage, we model the smoothed version of the data distribution using an autoregressive model.
In the second stage, we reverse the smoothing process---a procedure which can also be understood as ``denoising"---by either applying a gradient-based denoising approach~\citep{alain2014regularized} or introducing another conditional autoregressive model to recover the original data distribution from the smoothed one. 
By choosing an appropriate smoothing distribution, 
we aim to make each step easier than the original learning problem: smoothing facilitates learning in the first stage by making the input distribution fully supported without sharp transitions in the density function; generating a sample given a noisy one is easier than generating a sample from scratch.

We show with extensive experimental results that our approach is able to drastically improve the sample quality of current autoregressive models on several synthetic datasets and real-world image datasets, while obtaining competitive likelihoods on synthetic datasets. We empirically demonstrate that our method can also be applied to density estimation, image inpainting, and image denoising.
\begin{figure}
    \vspace{-16pt}
    \centering
    \includegraphics[width=0.75\textwidth]{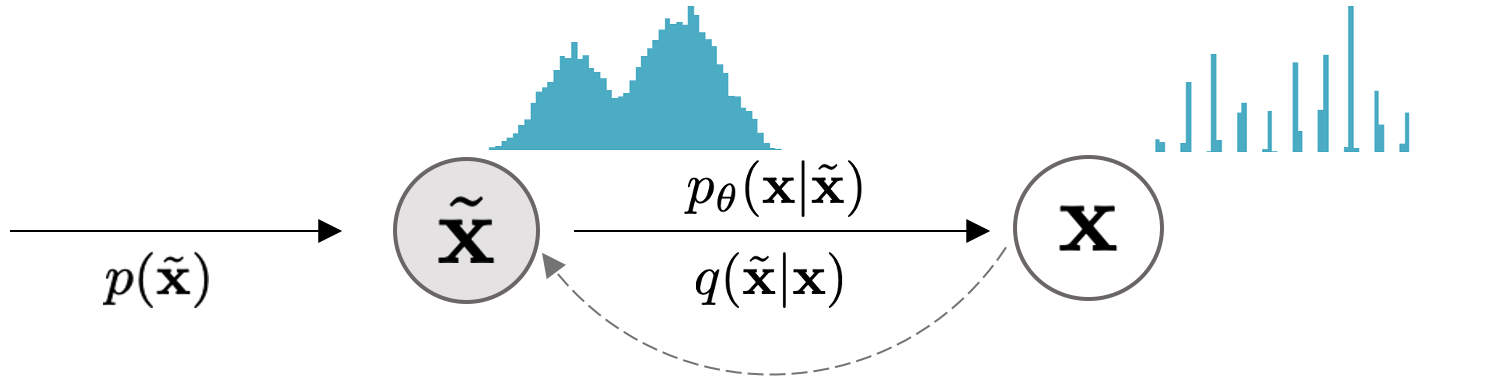}
    \caption{Overview of our method. From a data distribution ($\bfx$) we inject noise ($q(\tilde\bfx | \bfx)$) which makes the distribution smoother ($\tilde{\bfx}$); then we model the smoothed distribution ($p_\theta(\tilde{\bfx})$) as well as the denoising step ($p_\theta(\bfx | \tilde{\bfx})$), forming a two-step model. 
    }
    \label{fig:model}
\end{figure}

\section{Background}

We consider a density estimation problem. Given $D$-dimensional i.i.d samples $\{\bfx_1,\bfx_2,...,\bfx_N\}$ from a continuous data distribution
$p_{\text{data}}(\bfx)$, the goal is to approximate $p_{\text{data}}(\bfx)$ with a model $p_{\theta}(\bfx)$ parameterized by $\theta$. A commonly used approach for density estimation is maximum likelihood estimation (MLE), where the objective is to maximize %
    $L(\theta)\triangleq \frac{1}{N} \sum_{i=1}^{N}\log p_{\theta}(\bfx_i)$.

\subsection{Autoregressive models}
An autoregressive model~\citep{larochelle2011neural,salimans2017pixelcnn++} 
decomposes a joint distribution $p_{\theta}(\bfx)$ into the product of  univariate conditionals: %
\begin{equation}
    p_{\theta}(\bfx)=\prod_{i=1}^{D} p_{\theta}(x_i|\bfx_{<i}),
\end{equation}
where $x_i$ stands for the $i$-th component of $\bfx$, and $\bfx_{<i}$ refers to the components with indices smaller than $i$.
In general, an autoregressive model parameterizes each conditional $p_{\theta}(x_i|\bfx_{<i})$ using a pre-specified density function (\eg mixture of logistics). This bounds the capacity of the model by limiting the number of modes for each conditional. 

Although autoregressive models have achieved top likelihoods amongst all types of density based models, 
their sample quality is still lacking compared to energy-based models~\citep{du2019implicit} and score-based models~\citep{song2019generative}.
We believe this can be caused by the following two reasons.

\subsection{Manifold Hypothesis}
Several existing methods~\citep{roweis2000nonlinear,tenenbaum2000global} 
rely on the manifold hypothesis,
\textit{i.e}. that real-world high-dimensional data tends to lie on a low-dimensional manifold~\citep{narayanan2010sample}. If the manifold hypothesis is true, then the density of the data distribution is not
well defined in the ambient space; if the manifold hypothesis 
 holds only approximately 
and the data lies in the vicinity of a manifold, then 
only points that are very close to the manifold would have high density, while all other points would have close to zero density.  
Thus we may expect the data density
around the manifold to have large first-order derivatives, i.e. the density function has a high Lipschitz constant (if not infinity).

To see this, let us consider a 2-d example where the data distribution is a thin ring distribution (almost a unit circle) formed by rotating the 1-d Gaussian distribution $\mathcal{N}(1, 0.01^2)$ around the origin. 
The density function of the ring has a high Lipschitz constant near the ``boundary". 
Let us focus on a data point travelling along the diagonal as shown in the leftmost panel in \figref{fig:manifold_hypothesis}. We plot the first-order directional derivatives of the density for the point as it approaches the boundary from the inside, then lands on the ring, and finally moves outside the ring (see \figref{fig:manifold_hypothesis}). As we can see, when the point is far from the boundary, the derivative has a small magnitude. When the point moves closer to the boundary, the magnitude increases and changes significantly near the boundary even with small displacements in the trajectory. However, once the point has landed on the ring, the magnitude starts to decrease. As it gradually moves off the ring, the magnitude first increases and then decreases 
just like when the point approached the boundary from the inside.
It has been observed that certain likelihood models, such as normalizing flows, exhibit pathological behaviors on data distributions whose densities have high Lipschitz constants~\citep{cornish2019relaxing}. Since each conditional of an autoregressive model is a 1-d normalizing flow given a fixed context, a high Lipschitz constant on data density could also hinder learning of autoregressive models. 

\begin{figure}
\vspace{-16pt}
\centering
\includegraphics[width=0.98\textwidth]{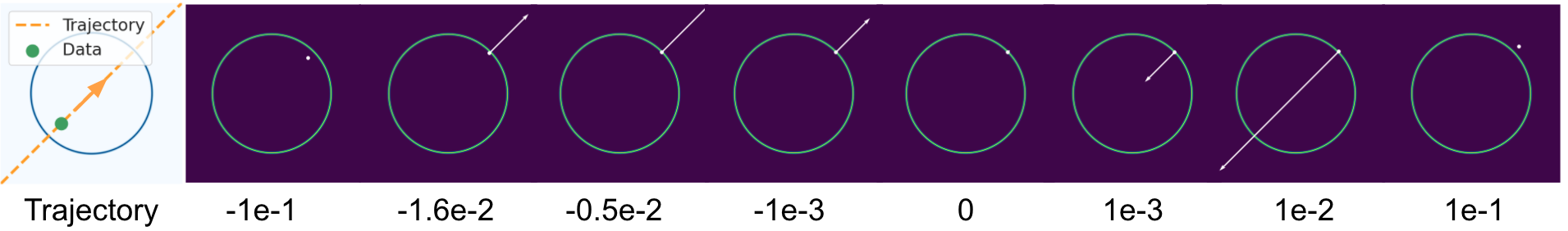}
\caption{Manifold hypothesis illustration. The data point is travelling along the diagonal as shown in the leftmost panel. The white arrow stands for the direction and magnitude of the derivative of density at the data point. The data location for each figure is $(\sqrt{0.5}+c, \sqrt{0.5}+c)$, where $c$ is the number below each figure and $(\sqrt{0.5}, \sqrt{0.5})$ is the upper right intersection of the trajectory with the unit circle.
} 
\vspace{-10pt}
\label{fig:manifold_hypothesis}
\end{figure}

\subsection{Compounding errors in autoregressive modeling}
Autoregressive models can also be  susceptible to compounding errors from the conditional distributions~\citep{lamb2016professor} during sampling time. 
We notice that an autoregressive model $p_{\theta}(\bfx)$ learns the joint density $p_{\text{data}}(\bfx)$ by matching each of the conditional $p_{\theta}(x_i|\bfx_{<i})$ with $p_{\text{data}}(x_i|\bfx_{<i})$.
In practice, we typically have access to a limited amount of training data,
which makes it hard for an autoregressive model to capture all the conditional distributions correctly due to the curse of dimensionality. 
During sampling, since a prediction is made based on the previously generated context, once a mistake is made at a previous step,
the model is likely to make more mistakes in the later steps, eventually generating a sample $\hat\bfx$ that is far from being an actual image, but is mistakenly assigned a high-likelihood by the model. 

The generated image $\hat\bfx$, being unrealistic but assigned a high-likelihood, 
resembles an adversarial example, i.e., an input that causes the model to make mistakes.
Recent works~\citep{cohen2019certified} in adversarial defense have shown that random noise can be used to improve the model's robustness to adversarial perturbations --- a process during which adversarial examples that are close to actual data are generated to fool the model. %
We hypothesize that such approach can also be applied to improve an autoregressive modeling process by making the model less vulnerable to compounding errors occurred during density estimation.
Inspired by the success of randomized smoothing in adversarial defense~\citep{cohen2019certified}, we propose to apply smoothing to autoregressive modeling to address the problems mentioned above.

\section{Generative models with distribution smoothing}
\begin{figure}
\vspace{-6pt}
    \centering
    \begin{subfigure}{0.35\linewidth}
    \includegraphics[width=\textwidth]{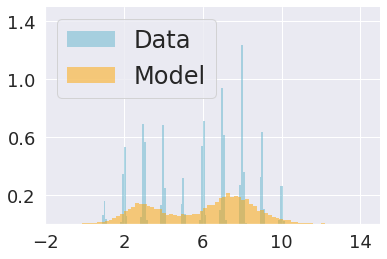}
    \caption{Data}
    \label{fig:1d_data}
    \end{subfigure}
    ~
    \begin{subfigure}{0.36\linewidth}
    \includegraphics[width=\textwidth]{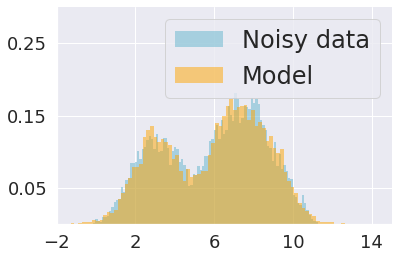}
    \caption{Smoothed data}
    \label{fig:1d_exp_noise}
    \end{subfigure}
\caption{Visualization of a 1-d data distribution without smoothing (a) or with smoothing (b), modeled by the same mixture of logistics model. 
} 
\label{fig:1d_data_example}
\end{figure}

In the following, we propose to decompose a density estimation task into a smoothed data modeling problem followed by an inverse smoothing problem where we recover the true data density from the smoothed one. 
\subsection{Randomized smoothing process}

Unlike \cite{cohen2019certified} where
randomized 
smoothing is applied to a model, we apply 
smoothing directly to the data distribution $p_{\text{data}}(\bfx)$. To do this, we introduce a smoothing distribution $q(\tilde \bfx|\bfx)$ --- a distribution that is symmetric and stationary (\eg a Gaussian or Laplacian kernel) --- 
and convolve it with $p_{\text{data}}(\bfx)$ to obtain a new distribution $q(\tilde \bfx)\triangleq \int q(\tilde \bfx|\bfx)p_{\text{data}}(\bfx) d\bfx$. When $q(\tilde \bfx|\bfx)$ is a normal distribution, this convolution process is equivalent to perturbing the data distribution with Gaussian noise, which, intuitively, will make the data distribution smoother.
In the following, we formally prove  that convolving a 1-d distribution $p_{\text{data}}(x)$ with a suitable noise can indeed ``smooth" $p_{\text{data}}(x)$.

\begin{restatable}{theorem}{lipschitz}
\label{thm:1d_lschitz}
Given a continuous and bounded 1-d distribution $p_{\text{data}}(x)$ that is supported on $\mathbb{R}$, for any 1-d distribution $q(\tilde x|x)$ that is symmetric (\ie  $q(\tilde x|x)=q(x|\tilde x)$), stationary (\ie  translation invariant) and satisfies $\lim_{x\to\infty}p_{\text{data}}(x)q(x|\tilde x)=0$ for any given $\tilde x$, we have $Lip(q(\tilde x))\le Lip(p_{\text{data}}(x))$, where $q(\tilde x)\triangleq \int q(\tilde x|x)p_{\text{data}}(x) dx$ and $Lip(\cdot)$ denotes the Lipschitz constant of the given 1-d function. 
\end{restatable}

\thmref{thm:1d_lschitz} shows that  convolving a 1-d data distribution $p_{\text{data}}(x)$ with a suitable noise distribution $q(\tilde x|x)$ (\eg $\mathcal{N}(\tilde x|x, \sigma^2)$) can reduce the Lipschitzness (\ie increase the smoothness) of $p_{\text{data}}(x)$. We provide the proof of \thmref{thm:1d_lschitz} in Appendix~\ref{app:proofs}. 

Given $p_{\text{data}}(\bfx)$ with a high Lipschitz constant, we empirically verify that density estimation becomes an easier task on the smoothed distribution $q(\tilde{\bfx})$ than directly on $p_{\text{data}}(\bfx)$.
To see this, we visualize a 1-d example in \figref{fig:1d_data}, where we want to model a ten-mode data distribution with a mixture of logistics model. 
If our model has three logistic components, there is almost no way for the model, which only has three modes, to perfectly fit this data distribution, which has ten separate modes with sharp transitions. The model, after training (see \figref{fig:1d_data}), mistakenly assigns a much higher density to the low density regions between nearby modes. If we convolve the data distribution with $q(\tilde x|x)=\mathcal{N}(\tilde x|x, 0.5^2)$, the new distribution becomes smoother (see \figref{fig:1d_exp_noise}) and can be captured reasonably well by the same mixture of logistics model with only three modes (see \figref{fig:1d_exp_noise}).
Comparing the same model's performance on the two density estimation tasks, we can see that the model is doing a better job at modeling the smoothed version of the data distribution than the original data distribution, which has a high Lipschitz constant.

This smoothing process can also be understood as a regularization term for the original maximum likelihood objective (on the un-smoothed data distribution), encouraging the learned model to be smooth, as formalized by the following statement:
\newcommand{\diff}{\mathrm{d}}
\begin{restatable}[Informal]{proposition}{regular}
\label{thm:regular} Assume that the symmetric and stationary smoothing distribution $q(\tilde\bfx|\bfx)$ has small variance and negligible higher order moments, then %
\begin{equation*}
\small
    \E_{p_{\text{data}}(\bfx)}\E_{ q(\tilde\bfx|\bfx)}[\log p_{\theta}(\tilde\bfx)] \approx \E_{p_{\text{data}}(\bfx)}\left[\log p_\theta(\bfx) + \frac{\eta}{2} \sum_{i} \frac{\partial^2 \log p_\theta}{\partial x_i^2}\right],
\end{equation*}
for some constant $\eta$.
\end{restatable}
Proposition \ref{thm:regular} shows that our smoothing process provides a regularization effect on the original objective $\E_{p_{\text{data}}(\bfx)}[\log p_\theta(\bfx)]$ when no noise is added, where the regularization aims to maximize $\frac{\eta}{2} \sum_{i} \frac{\partial^2 \log p_\theta}{\partial x_i^2}$. Since samples from $\pdata$ should be close to a local maximum of the model, this encourages the second order gradients computed at a data point $\bfx$ to become closer to zero (if it were positive then $\bfx$ will not be a local maximum), creating a smoothing effect.
This extra term is also the trace of the score function (up to a multiplicative constant) that can be found in the score matching objective~\citep{hyvarinen2005estimation}, which is closely related to many denoising methods~\citep{vincent2011connection,hyvarinen2008optimal}.
This regularization effect can, intuitively, increase the generalization capability of the model.
In fact, it has been demonstrated empirically that training with noise can lead to improvements in network generalization \citep{sietsma1991creating,bishop1995training}. 
Our argument is also similar to that used in \citep{bishop1995training} except that we consider a more general generative modeling case as opposed to supervised learning with squared error. We provide the formal statement and proof of Proposition \ref{thm:regular} in Appendix~\ref{app:proofs}.

\subsection{Autoregressive distribution smoothing models}
Motivated by the previous 1-d example, instead of directly modeling $p_{\text{data}}(\bfx)$, which can have a high Lipschitz constant, we propose to first train an autoregressive model on the smoothed version of the data distribution $q(\tilde \bfx)$. Although the smoothing process makes the distribution easier to learn, it also introduces bias. Thus, we need an extra step to debias the learned distribution by reverting the smoothing process.

If our goal is to generate approximate samples for $\pdata(\bfx)$, 
when $q(\tilde\bfx|\bfx)=\mathcal{N}(\tilde\bfx|\bfx, \sigma^2I)$ 
and $\sigma$ is small, we can use the gradient of $p_{\theta}(\tilde \bfx)$ for denoising~\citep{alain2014regularized}.
More specifically, given smoothed samples $\tilde \bfx$ from $p_{\theta}(\tilde\bfx)$, we can ``denoise" samples via:
\begin{equation}
\label{eq:single_step}
    \bar{\bfx}=\tilde\bfx+\sigma^2\nabla_{\tilde\bfx}\log p_{\theta}(\tilde\bfx),
\end{equation}
which only requires the knowledge of $p_{\theta}(\tilde\bfx)$ and the ability to sample from it.
However, this approach does not provide a likelihood estimate 
and Eq.~(\ref{eq:single_step}) only works when $q(\tilde\bfx|\bfx)$ is Gaussian (though alternative denoising updates for other smoothing processes could be derived under the Empirical Bayes framework~\citep{raphan2011least}).
Although Eq.~(\ref{eq:single_step}) could provide reasonable denoising results when the smoothing distribution has a small variance, 
$\bar{\bfx}$ obtained in this way is only a point estimation of 
$\bar{\bfx}=\mathbb{E}[\bfx|\tilde{\bfx}]$ 
and does not capture the uncertainty of $p(\bfx|\tilde{\bfx})$.

To invert more general smoothing distributions (beyond Gaussians)
 and to obtain likelihood estimations, we introduce a second autoregressive model $p_{\theta}(\bfx|\tilde \bfx)$. The parameterized joint density $p_{\theta}(\bfx, \tilde{\bfx})$ can then be computed as $p_{\theta}(\bfx, \tilde{\bfx})=p_{\theta}(\bfx|\tilde {\bfx})p_{\theta}(\tilde{\bfx})$. 
To obtain our approximation of $p_{\text{data}}(\bfx)$, we need to integrate over $\tilde {\bfx}$ on the joint distribution $p_{\theta}(\bfx, \tilde{\bfx})$ to obtain $p_{\theta}(\bfx)=\int p_{\theta}(\bfx,\tilde\bfx)d\tilde \bfx$, which is in general intractable. However, we can easily obtain an evidence lower bound (ELBO):
\begin{equation}
    \log p_{\theta}(\bfx)\ge \mathbb{E}_{q(\tilde \bfx|\bfx)}[\log p_{\theta}(\tilde \bfx)]-\mathbb{E}_{q(\tilde \bfx|\bfx)}[\log q(\tilde \bfx|\bfx)] + \mathbb{E}_{q(\tilde \bfx|\bfx)}[\log p_{\theta}(\bfx|\tilde \bfx)].
\end{equation}

Note that when $q(\tilde \bfx|\bfx)$ is fixed, the entropy term $\mathbb{E}_{q(\tilde \bfx|\bfx)}[\log q(\tilde \bfx|\bfx)]$ is a constant with respect to the optimization parameters. Maximizing ELBO on $p_{\text{data}}(\bfx)$ is then equivalent to maximizing:
\begin{equation}
    J(\theta)=\mathbb{E}_{p_{\text{data}}(\bfx)}\bigg[\mathbb{E}_{q(\tilde \bfx|\bfx)}[\log p_{\theta}(\tilde \bfx)]+\mathbb{E}_{q(\tilde \bfx|\bfx)}[\log p_{\theta}(\bfx|\tilde \bfx)]\bigg].
    \label{eq:objective_elbo}
\end{equation}
From \eqref{eq:objective_elbo}, we can see that optimizing the two models $p_{\theta}(\tilde \bfx)$ and $p_{\theta}(\bfx|\tilde \bfx)$ separately via maximum likelihood estimation is equivalent to optimizing $J(\theta)$.

\subsection{Tradeoff in modeling}
\label{sec:3.3}
In general, there is a trade-off between the difficulty of modeling $p_{\theta}(\tilde \bfx)$ and $p_{\theta}(\bfx|\tilde \bfx)$.
To see this, let us consider two extreme cases for the variance of $q(\tilde \bfx|\bfx)$ --- when $q(\tilde \bfx|\bfx)$ has a zero variance and an infinite variance. When $q(\tilde \bfx|\bfx)$ has a zero variance, $q(\tilde \bfx|\bfx)$ is a distribution with all its probability mass at $\bfx$, meaning that no noise is added to the data distribution. 
In this case, modeling the smoothed distribution would be equivalent to modeling $p_{\text{data}}(\bfx)$, which can be hard as discussed above. 
The reverse smoothing process, however, would be easy since $p_{\theta}(\bfx|\tilde \bfx)$ can simply be an identity map to perfectly invert the smoothing process. In the second case when $q(\tilde \bfx|\bfx)$ has an infinite variance, modeling $p(\tilde \bfx)$ would be easy  because all the information about the original data is lost, and
$p(\tilde \bfx)$ 
would be close to the smoothing distribution. 
Modeling $p(\bfx|\tilde \bfx)$, on the other hand, is equivalent to directly modeling $p_{\text{data}}(\bfx)$, which can be challenging.
\begin{figure}
\vspace{-8pt}
    \centering
    \includegraphics[width=\textwidth]{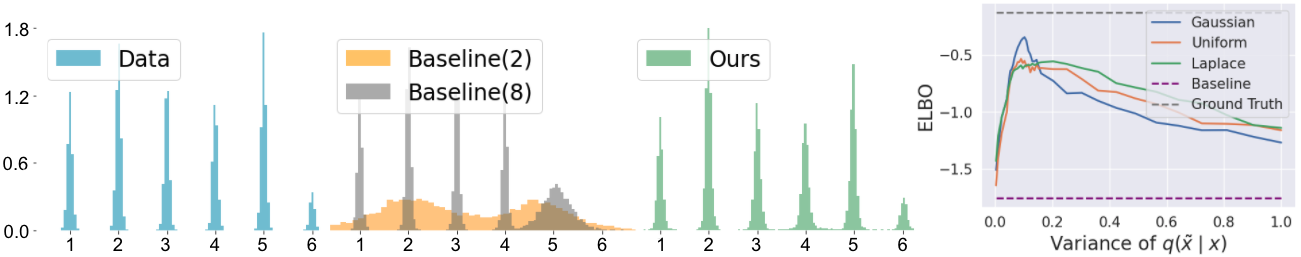}
    \caption{Density estimation on 1-d synthetic dataset. In the second figure, the digit in the parenthesis denotes the number of mixture components used in the baseline mixture of logistics model. In comparison, our model in the third figure uses only 2 mixture of logistics components for each univariate conditional distribution.
    } 
    \vspace{-6pt}
    \label{fig:1d_examples}
\end{figure}  

Thus, the key here is to appropriately choose a smoothing level so that both $q(\tilde \bfx)$ and $p(\bfx|\tilde \bfx)$ can be approximated relatively well by existing autoregressive models. In general, the optimal variance might be hard to find. 
Although one can train $q(\tilde \bfx|\bfx)$ by jointly optimizing ELBO, in practice, we find this approach often assigns a very large variance to $q(\tilde \bfx|\bfx)$, which can trade-off sample quality for better likelihoods on high dimensional image datasets. We find empirically that a pre-specified $q(\tilde \bfx|\bfx)$ chosen by heuristics~\citep{saremi2019neural,garreau2017large} is able to generate much better samples than training $q(\tilde \bfx|\bfx)$ via ELBO. 
In this paper, we will focus on the sample quality and leave the training of $q(\tilde \bfx|\bfx)$ for future work.

\section{Experiments}
In this section, we demonstrate empirically that by appropriately choosing the smoothness level of randomized smoothing, our approach is able to drastically improve the sample quality of existing autoregressive models on several synthetic and real-world datasets while retaining competitive likelihoods on synthetic datasets. We also present results on image inpainting in Appendix~\ref{app:image_inpainting}.

\subsection{Choosing the smoothing distribution}
To help us build insights into the selection of the smoothing distribution $q(\tilde\bfx| \bfx)$, we first focus on a 1-d multi-modal distribution (see \figref{fig:1d_examples} leftmost panel). We use model-based methods to invert the smoothed distribution and provide analysis on ``single-step denoising" in Appendix \ref{app:1-d}.
We start with the exploration of three different types of smoothing distributions -- Gaussian distribution, Laplace distribution, and uniform distribution. For each type of distribution, we perform a grid search to find the optimal variance.
Since our approach requires the modeling of both $p_{\theta}(\tilde\bfx)$ and $p_{\theta}(\bfx|\tilde\bfx)$, we stack $\tilde\bfx$ and $\bfx$ together, and use a MADE model \citep{germain2015made} with a mixture of two logistic components to parameterize $p_{\theta}(\tilde\bfx)$ and $p_{\theta}(\bfx|\tilde\bfx)$ at the same time. 
For the baseline model, we train a mixture of logistics model directly on $p_{\text{data}}(\bfx)$. We compare the results in the middle two panels in \figref{fig:1d_examples}.

We find that although the baseline with eight logistic components has the capacity to perfectly model the multi-modal data distribution, which has six modes, the baseline model still fails to do so. We believe this can be caused by optimization or initialization issues for modeling a distribution with a high Lipschitz constant. Our method, on the other hand, demonstrates more robustness by successfully modeling the different modes in the data distribution even when using only two mixture components for both $p_{\theta}(\tilde\bfx)$ and $p_{\theta}(\bfx|\tilde\bfx)$.

For all the three types of smoothing distributions, we observe a reverse U-shape correlation between the variance of $q(\tilde\bfx|\bfx)$ and ELBO values --- with ELBO first increasing as the variance increases and then decreasing as the variance grows beyond a certain point. The results match our discussion on the trade-off between modeling $p_{\theta}(\tilde\bfx)$ and $p_{\theta}(\bfx|\tilde\bfx)$ in Section~\ref{sec:3.3}. 
We notice from the empirical results that Gaussian smoothing is able to obtain better ELBO than the other two distributions. Thus, we will use $q(\tilde \bfx|\bfx)=\mathcal{N}(\tilde \bfx|\bfx, \sigma^2I)$ for the later experiments.

\begin{table}
\begin{center}
\caption{Negative log-likelihoods on 2-d synthetic datasets (lower is better). We compare with MADE~\citep{germain2015made}, RealNVP~\citep{dinh2016density}, CIF-RealNVP~\citep{cornish2019relaxing}. 
}
\resizebox{\textwidth}{!}{
    \begin{tabular}{lccccc}
        \toprule
        Dataset &RealNVP &CIF-RealNVP &MADE (3 mixtures) &MADE (6 mixtures) &Ours (3 mixtures)\\
        \midrule
        Rings &$2.81$  &$2.81$ &$3.26$ &$2.81$ &$\mathbf{2.71}$\\
        Olympics &$1.80$ &$1.74$ &$1.27$ &$\mathbf{0.80}$ &$\mathbf{0.80}$\\
        \bottomrule
    \end{tabular}
    }
    \label{tab:2d_likelihood}
\end{center}
\end{table}

\subsection{2-D synthetic datasets}
\label{sec:2-d-dataset}
\vspace{-10pt}
\begin{figure}
    \centering
    \begin{subfigure}{0.16\textwidth}
    \includegraphics[width=\textwidth]{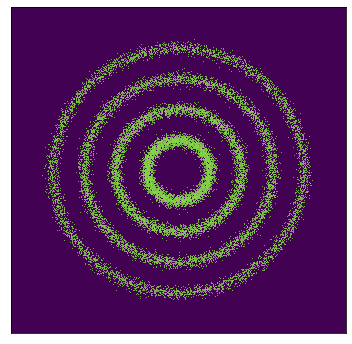}
    \caption{Rings}
    \label{fig:2d_data}
    \end{subfigure}
    \begin{subfigure}{0.16\textwidth}
    \includegraphics[width=\textwidth]{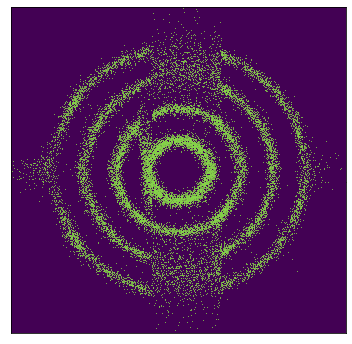}
    \caption{MADE (6)}
    \label{fig:2d_baseline}
    \end{subfigure}
    \begin{subfigure}{0.16\textwidth}
    \includegraphics[width=\textwidth]{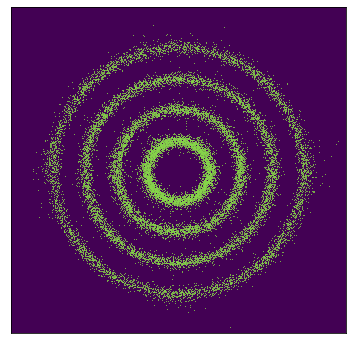}
    \caption{Ours (3)}
    \label{fig:2d_denoised_samples}
    \end{subfigure}
    \begin{subfigure}{0.16\textwidth}
    \includegraphics[width=\textwidth]{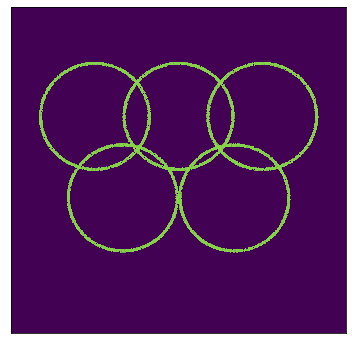}
    \caption{Olympics}
    \label{fig:2d_data_2}
    \end{subfigure}
    \begin{subfigure}{0.16\textwidth}
     \includegraphics[width=\textwidth]{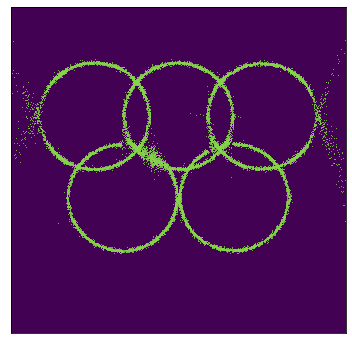}
    \caption{MADE (6)}
    \label{fig:2d_baseline_2}
    \end{subfigure}
    \begin{subfigure}{0.16\textwidth}
    \includegraphics[width=\textwidth]{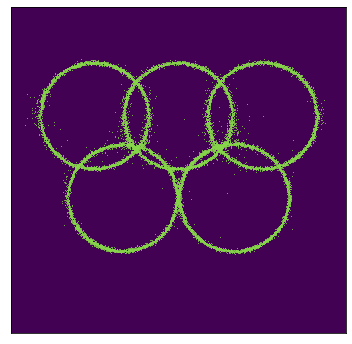}
    \caption{Ours (3)}
    \label{fig:2d_noisy_samples_2}
    \end{subfigure}
    \caption{Samples on 2-d synthetic datasets. We use a MADE model with comparable number of parameters for both our method and the baseline. Our model uses 3 mixture of logistics, while the baseline uses 6 ({more}) mixture of logistics.
    }
    \label{fig:2d_example}
\end{figure}

In this section, we consider two challenging 2-d multi-modal synthetic datasets (see \figref{fig:2d_example}). We focus on model-based denoising methods and present discussion on ``single-step denoising" in Appendix~\ref{app:2-d_denoise}.
We use a MADE model with comparable number of total parameters for both the baseline and our approach. 
For the baseline, we train the MADE model directly on the data. For our randomized smoothing model, we choose $q(\tilde \bfx|\bfx)=\mathcal{N}(\tilde \bfx|\bfx, 0.3^2I)$ to be the smoothing distribution. We observe that with this randomized smoothing approach, our model is able to generate better samples than the baseline (according to a human observer) 
even when using less logistic components (see \figref{fig:2d_example}). We provide more analysis on the model's performance in Appendix~\ref{app:2-d_denoise}. 
We also provide the negative log-likelihoods in \tabref{tab:2d_likelihood}.

\subsection{Image experiments}
In this section, we focus on three common image datasets, namely MNIST, CIFAR-10~\citep{krizhevsky2009learning} and CelebA~\citep{liu2015deep}. We select $q(\tilde \bfx|\bfx)=\mathcal{N}(\tilde\bfx|\bfx,\sigma^2I)$ to be the smoothing distribution. We use PixelCNN++~\citep{salimans2017pixelcnn++} as the model architecture for both $p_{\theta}(\tilde \bfx)$ and $p_{\theta}(\bfx|\tilde \bfx)$. We provide more details about settings in Appendix~\ref{app:image}.

\begin{figure}[t]
\vspace{-2pt}
    \centering
    \begin{subfigure}{0.23\textwidth}
    \adjincludegraphics[width=\textwidth,trim={0 {0.6\height} {0.57\width} 0},clip]{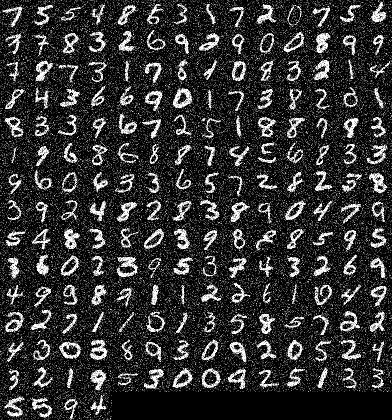}
    \end{subfigure}
    ~
    \begin{subfigure}{0.23\textwidth}
    \adjincludegraphics[width=\textwidth,trim={0 {0.6\height} {0.57\width} 0},clip]{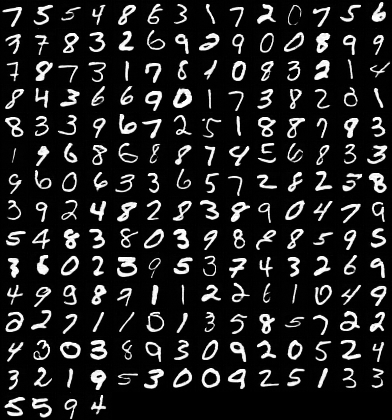}
    \end{subfigure}
    ~
    \begin{subfigure}{0.23\textwidth}
    \adjincludegraphics[width=\textwidth,trim={0 {0.6\height} {0.57\width} 0},clip]{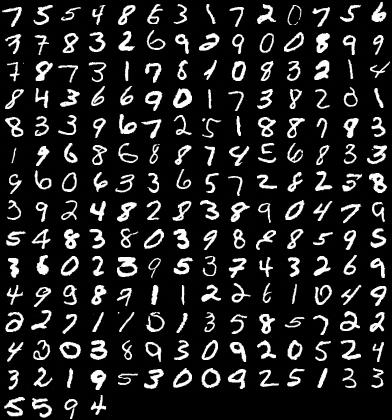}
    \end{subfigure}
    ~
    \begin{subfigure}{0.23\textwidth}
    \adjincludegraphics[width=\textwidth,trim={0 {0.6\height} {0.57\width} 0},clip]{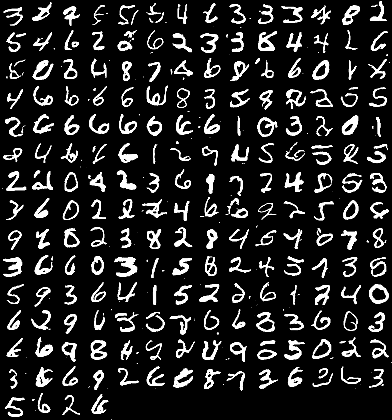}
    \end{subfigure}
    \hfill
    \begin{subfigure}{0.23\textwidth}
    \adjincludegraphics[width=\textwidth,trim={0 {0.25\height} {0.25\width} 0},clip]{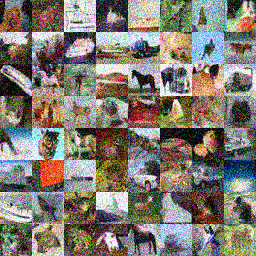}
    \end{subfigure}
    ~
    \begin{subfigure}{0.23\textwidth}
    \adjincludegraphics[width=\textwidth,trim={0 {0.25\height} {0.25\width} 0},clip]{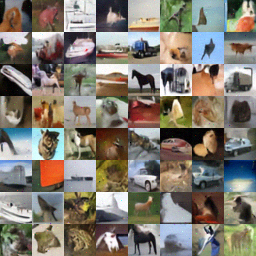}
    \end{subfigure}
    ~
    \begin{subfigure}{0.23\textwidth}
    \adjincludegraphics[width=\textwidth,trim={0 {0.25\height} {0.25\width} 0},clip]{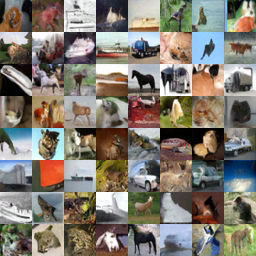}
    \end{subfigure}
    ~
    \begin{subfigure}{0.23\textwidth}
    \adjincludegraphics[width=\textwidth,trim={0 {0.25\height} {0.25\width} 0},clip]{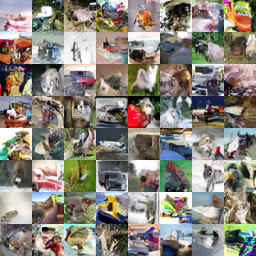}
    \end{subfigure}
    \hfill
    \begin{subfigure}{0.23\textwidth}
    \adjincludegraphics[width=\textwidth,trim={0 {0.25\height} {0.25\width} 0},clip]{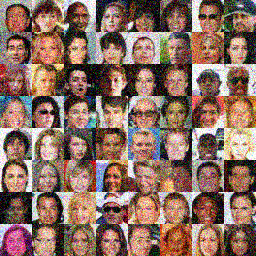}
    \caption*{Column 1}
    \end{subfigure}
    ~
    \begin{subfigure}{0.23\textwidth}
    \adjincludegraphics[width=\textwidth,trim={0 {0.25\height} {0.25\width} 0},clip]{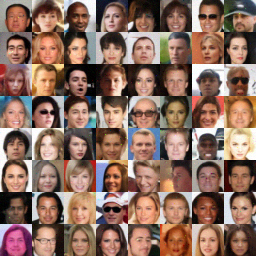}
    \caption*{Column 2}
    \end{subfigure}
    ~
    \begin{subfigure}{0.23\textwidth}
    \adjincludegraphics[width=\textwidth,trim={0 {0.25\height} {0.25\width} 0},clip]{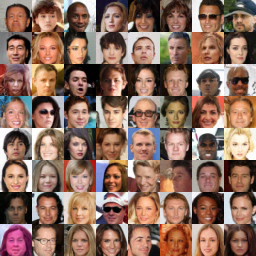}
    \caption*{Column 3}
    \end{subfigure}
    ~
    \begin{subfigure}{0.23\textwidth}
    \adjincludegraphics[width=\textwidth,trim={0 {0.57\height} {0.5\width} 0},clip]{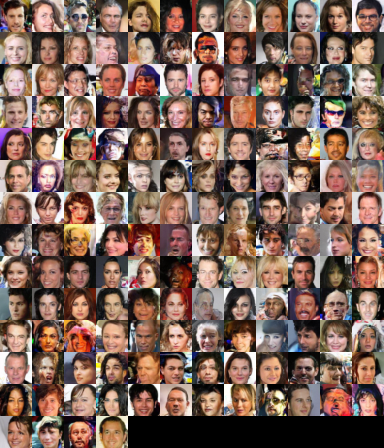}
    \caption*{Column 4}
    \end{subfigure}
    \caption{From left to right: \textbf{Column 1:} samples from $p_{\theta}(\tilde\bfx)$. \textbf{Column 2:} ``single-step denoising" samples from $p_{\theta}(\tilde\bfx)$. \textbf{Column 3:} samples from $p_{\theta}(\bfx|\tilde \bfx)$. \textbf{Column 4:} samples from the baseline PixelCNN++ model. 
    Samples in Column 2 (``single-step denoising") contain wild pixels and are less realistic compared to samples in Column 3 (modeled by another PixelCNN++).
    None of the samples are conditioned on class labels. %
    }
    \vspace{-6pt}
    \label{fig:samples}
\end{figure}

\textbf{Image generation.}
For image datasets, we select the $\sigma$ of $q(\tilde \bfx|\bfx)=\mathcal{N}(\tilde \bfx|\bfx, \sigma^2I)$ according to analysis in \citep{saremi2019neural} (see Appendix~\ref{app:image} for more details). 
Since $q(\tilde\bfx|\bfx)$ is a Gaussian distribution, we can apply ``single-step denoising" to reverse the smoothing process for samples drawn from $p_{\theta}(\tilde\bfx)$. In this case, the model $p_{\theta}(\bfx|\tilde\bfx)$ is not required for sampling since the gradient of $p_{\theta}(\tilde\bfx)$ can be used to denoise samples (also from  $p_{\theta}(\tilde\bfx)$) (see \eqref{eq:single_step}). 
We present smoothed samples from $p_{\theta}(\tilde\bfx)$, reversed smoothing samples processed by ``single-step denoising" and processed by  $p_{\theta}(\bfx|\tilde\bfx)$ in \figref{fig:samples}.
For comparison, we also present samples from a PixelCNN++ with parameters comparable to the sum of total parameters of $p_{\theta}(\tilde\bfx)$ and $p_{\theta}(\bfx|\tilde\bfx)$.
We find that by using this randomized smoothing approach, we are able to drastically improve the sample quality of PixelCNN++ (see the rightmost panel in \figref{fig:samples}).
We note that with only $p_{\theta}(\tilde\bfx)$, a PixelCNN++ optimized on the smoothed data, we already obtain more realistic samples compared to the original PixelCNN++ method. However, 
$p_{\theta}(\bfx|\tilde\bfx)$ is needed to compute the likelihood lower bounds.
We report the sample quality evaluated by Fenchel Inception Distance (FID~\citep{heusel2017gans}), Kernel Inception Distance (KID ~\citep{binkowski2018demystifying}), and Inception scores~\citep{salimans2016improved} in \tabref{tab:score}. 
Although our method obtains better samples compared to the original PixelCNN++, our model has worse likelihoods as evaluated in BPDs. We believe this is because likelihood and sample quality are not always directly correlated as discussed in \citet{theis2015note}. We also tried training the variance for $q(\tilde \bfx|\bfx)$ by jointly optimizing ELBO. Although training the variance can produce better likelihoods, it does not generate samples with comparable quality as our method (\ie choosing variance by heuristics). Thus, it is hard to conclusively determine what is the best way of choosing $q(\tilde \bfx|\bfx)$. We provide more image samples in Appendix \ref{app:image_samples} and nearest neighbors analysis in Appendix~\ref{app:neighbor}.

\begin{table}[H]
\begin{center}
\begin{adjustbox}{max width=\linewidth}
    \begin{tabular}{lcccc}
        \toprule
        Model & Inception $\uparrow$ & FID $\downarrow$ & KID $\downarrow$ &BPD $\downarrow$\\
        \midrule
        PixelCNN~\citep{oord2016pixel} & $4.60$ & $65.93$ &- & $3.14$\\
        PixelIQN~\citep{ostrovski2018autoregressive} & $5.29$ & $49.46$ &- &-\\
        PixelCNN++~\citep{salimans2017pixelcnn++} & $5.30$ & $54.25$ &0.046  & $\mathbf{2.92}$\\
        EBM~\citep{du2019implicit} & $6.02$ & $40.58$ &- &-\\
        i-ResNet~\citep{behrmann2019invertible} &- &65.01 &- & 3.45\\
        MADE~\citep{germain2015made} &- &- &- & 5.67\\
        Glow~\citep{kingma2018glow} &- &46.90 &- &3.35 \\
        Single-step (Ours) &$7.50\pm .08$ &57.53 &0.052 &- \\
        Two-step (Ours) &$\textbf{7.84} \pm .07$  &$\mathbf{29.83}$ &$\textbf{0.022}$ &$\le 3.53$\\
        \bottomrule
    \end{tabular} 
 \end{adjustbox}
\end{center}
\caption{Inception, FID and KID scores for unconditional CIFAR-10. ``Single-step" samples are generated solely by $p_{\theta}(\tilde\bfx)$.  ``Two-step" samples are generated by sampling from $p_{\theta}(\tilde\bfx)$ and then denoised by $p_{\theta}(\bfx|\tilde\bfx)$. Although samples from ``single-step" might appear visually similar to samples from the ``two-step" method, there is still a gap between their Inception, FID and KID scores. 
}
\vspace{-6pt}
\label{tab:score}
\end{table}

\section{Additional experiments on normalizing flows}
In this section, we demonstrate empirically on 2-d synthetic datasets that randomized smoothing techniques can also be applied to improve the sample quality of normalizing flow models~\citep{rezende2015variational}. We focus on RealNVP~\citep{dinh2016density}. 
We compare the RealNVP model trained with randomized smoothing, where we use $p_{\theta}(\bfx|\tilde\bfx)$ (also a RealNVP) to revert the smoothing process, with a RealNVP trained with the original method but with comparable number of parameters.
We observe that smoothing is able to improve sample quality on the datasets we consider (see \figref{fig:2d_flow}) while also obtaining competitive likelihoods. On the checkerboard dataset, our method has negative log-likelihoods 3.64 while the original RealNVP has 3.72; on the Olympics dataset, our method has negative log-likelihoods 1.32 while the original RealNVP has 1.80.  
This example demonstrates that randomized smoothing techniques can also be applied to normalizing flow models. 

\begin{figure}
\vspace{-16pt}
    \centering
    \begin{subfigure}{0.16\textwidth}
    \includegraphics[width=\textwidth]{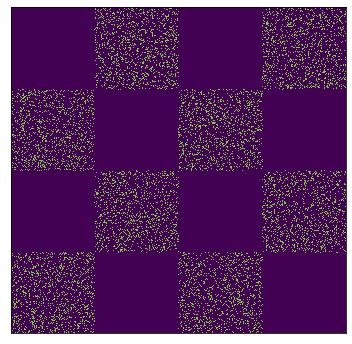}
    \caption{Data}
    \end{subfigure}
    \begin{subfigure}{0.16\textwidth}
    \includegraphics[width=\textwidth]{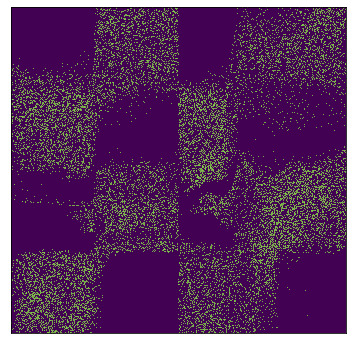}
    \caption{RealNVP}
    \label{fig:2d_denoised_samples}
    \end{subfigure}
    \begin{subfigure}{0.16\textwidth}
    \includegraphics[width=\textwidth]{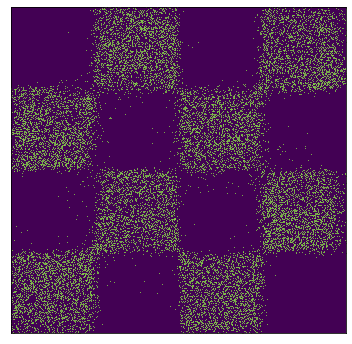}
    \caption{Ours}
    \end{subfigure}
    \begin{subfigure}{0.16\textwidth}
    \includegraphics[width=\textwidth]{figures/2d_visualization/olympics_data.png}
    \caption{Data}
    \end{subfigure}
    \begin{subfigure}{0.16\textwidth}
    \includegraphics[width=\textwidth]{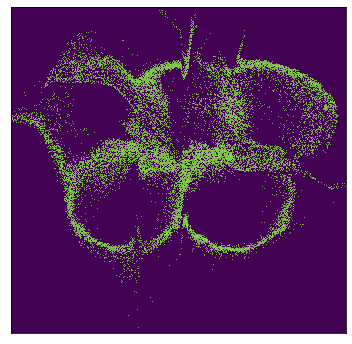}
    \caption{RealNVP}
    \end{subfigure}
    \begin{subfigure}{0.16\textwidth}
    \includegraphics[width=\textwidth]{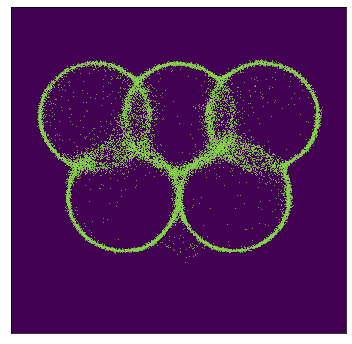}
    \caption{Ours}
    \end{subfigure}
    \vspace{-6pt}
    \caption{RealNVP samples on 2-d synthetic datasets. The RealNVP model trained with randomized smoothing is able to generate better samples according to human observers.
    }
    \label{fig:2d_flow}
\end{figure}
\section{Related Work}
Our approach shares some similarities with denoising autoencoders~(DAE,~\cite{vincent2008extracting}) which recovers a clean observation from a corrupted one. However, unlike DAE which has a trainable encoder and a fixed prior distribution, our approach fixes the encoder and models the prior using an autoregressive model. Generative stochastic networks (GSN,~\cite{bengio2014deep}) use DAEs to train a Markov chain whose equilibrium distribution matches the data distribution. However, GSN needs to start the chain from a sample that is very close to the training distribution. %
Denoising diffusion model~\citep{sohl2015deep,ho2020denoising} and NCSN~(\cite{song2019generative,song2020improved}) address the issue of GSNs by considering a sequence of distributions corresponding to data corrupted with various noise levels. By setting multiple noise levels that are close to each other, the sample from the previous level can serve as a proper initialization for the sampling process at the next level. This way, the model can start from a distribution that is easy to model and gradually move to the desired distribution. However, due to the large number of noise levels, such approaches require many steps for the chain to converge to the right data distribution. 

In this paper, we instead propose to use \textit{only one level} of smoothing by modeling each step with a powerful autoregressive model instead of deterministic autoencoders. Motivated by the success of ``randomized smoothing" techniques in adversarial defense~\citep{cohen2019certified}, we perform randomized smoothing directly on the data distribution. %
Unlike denoising score matching~\citep{vincent2011connection}, a technique closely related to denoising diffusion models and NCSN, which requires the perturbed noise to be a Gaussian distribution, we are able to work with different noise distributions. %

Our smoothing method is also relevant to ``dequantization'' approaches that are common in normalizing flow models, where the discrete data distribution is converted to a continuous one by adding continuous noise~\citep{uria2013rnade, ho2019flow++}. 
However the added noise for ``dequantization'' in flows is often indistinguishable to human eyes, and the reverse ``dequantization'' process is often ignored.
In contrast, we consider noise scales that are significantly larger and thus a denoising process is required.

Our method is also related to ``quantization" approaches which reduce the number of ``significant" bits that are modeled by a generative model~\citep{kingma2018glow,menick2018generating}. For instance, Glow~\citep{kingma2018glow} only models the 5 most significant bits of an image, which improves the visual quality of samples but decreases color fidelity.  SPN~\citep{menick2018generating} introduces another network to predict the remaining bits conditioned on the 3 most significant bits already modeled. Modeling the most significant bits can be understood as capturing a data distribution perturbed by bit-wise correlated noise, similar to modeling  smoothed data in our method. Modeling the remaining bits conditioned on the most significant ones in SPN is then similar to denoising. 
However, unlike these quantization approaches which process an image at the ``significant" bits level,
we apply continuous data independent Gaussian noise to the entire image with a different motivation to smooth the data density function.

\section{Discussion}
In this paper, we propose to incorporate randomized smoothing techniques into autoregressive modeling. By choosing the smoothness level appropriately, this seemingly simple approach is able to drastically improve the sample quality of existing autoregressive models on several synthetic and real-world datasets while retaining reasonable likelihoods. Our work provides insights into how recent adversarial defense techniques can be leveraged to building more robust generative models. Since we apply randomized smoothing technique directly to the target data distribution other than the model, we believe our approach is also applicable to other generative models such as 
variational autoencoders (VAEs) and generative adversarial networks (GANs). %

\section*{Acknowledgements}
The authors would like to thank Kristy Choi for reviewing the draft of the paper. This research was supported by 
NSF (\#1651565, \#1522054, \#1733686), ONR (N00014-19-1-2145), AFOSR (FA9550-19-1-0024), ARO, and Amazon AWS.

\bibliography{iclr2021_conference}

\begin{thebibliography}{47}
\providecommand{\natexlab}[1]{#1}
\providecommand{\url}[1]{\texttt{#1}}
\expandafter\ifx\csname urlstyle\endcsname\relax
  \providecommand{\doi}[1]{doi: #1}\else
  \providecommand{\doi}{doi: \begingroup \urlstyle{rm}\Url}\fi

\bibitem[Alain \& Bengio(2014)Alain and Bengio]{alain2014regularized}
Guillaume Alain and Yoshua Bengio.
\newblock What regularized auto-encoders learn from the data-generating
  distribution.
\newblock \emph{The Journal of Machine Learning Research}, 15\penalty0
  (1):\penalty0 3563--3593, 2014.

\bibitem[Behrmann et~al.(2019)Behrmann, Grathwohl, Chen, Duvenaud, and
  Jacobsen]{behrmann2019invertible}
Jens Behrmann, Will Grathwohl, Ricky~TQ Chen, David Duvenaud, and
  J{\"o}rn-Henrik Jacobsen.
\newblock Invertible residual networks.
\newblock In \emph{International Conference on Machine Learning}, pp.\
  573--582, 2019.

\bibitem[Belkin \& Niyogi(2003)Belkin and Niyogi]{belkin2003laplacian}
Mikhail Belkin and Partha Niyogi.
\newblock Laplacian eigenmaps for dimensionality reduction and data
  representation.
\newblock \emph{Neural computation}, 15\penalty0 (6):\penalty0 1373--1396,
  2003.

\bibitem[Bengio et~al.(2014)Bengio, Laufer, Alain, and
  Yosinski]{bengio2014deep}
Yoshua Bengio, Eric Laufer, Guillaume Alain, and Jason Yosinski.
\newblock Deep generative stochastic networks trainable by backprop.
\newblock In \emph{International Conference on Machine Learning}, pp.\
  226--234, 2014.

\bibitem[Biggio et~al.(2013)Biggio, Corona, Maiorca, Nelson, {\v{S}}rndi{\'c},
  Laskov, Giacinto, and Roli]{biggio2013evasion}
Battista Biggio, Igino Corona, Davide Maiorca, Blaine Nelson, Nedim
  {\v{S}}rndi{\'c}, Pavel Laskov, Giorgio Giacinto, and Fabio Roli.
\newblock Evasion attacks against machine learning at test time.
\newblock In \emph{Joint European conference on machine learning and knowledge
  discovery in databases}, pp.\  387--402. Springer, 2013.

\bibitem[Bi{\'n}kowski et~al.(2018)Bi{\'n}kowski, Sutherland, Arbel, and
  Gretton]{binkowski2018demystifying}
Miko{\l}aj Bi{\'n}kowski, Dougal~J Sutherland, Michael Arbel, and Arthur
  Gretton.
\newblock Demystifying mmd gans.
\newblock \emph{arXiv preprint arXiv:1801.01401}, 2018.

\bibitem[Bishop(1995)]{bishop1995training}
Chris~M Bishop.
\newblock Training with noise is equivalent to tikhonov regularization.
\newblock \emph{Neural computation}, 7\penalty0 (1):\penalty0 108--116, 1995.

\bibitem[Cohen et~al.(2019)Cohen, Rosenfeld, and Kolter]{cohen2019certified}
Jeremy~M Cohen, Elan Rosenfeld, and J~Zico Kolter.
\newblock Certified adversarial robustness via randomized smoothing.
\newblock \emph{arXiv preprint arXiv:1902.02918}, 2019.

\bibitem[Cornish et~al.(2019)Cornish, Caterini, Deligiannidis, and
  Doucet]{cornish2019relaxing}
Rob Cornish, Anthony~L Caterini, George Deligiannidis, and Arnaud Doucet.
\newblock Relaxing bijectivity constraints with continuously indexed
  normalising flows.
\newblock \emph{arXiv preprint arXiv:1909.13833}, 2019.

\bibitem[Dinh et~al.(2016)Dinh, Sohl-Dickstein, and Bengio]{dinh2016density}
Laurent Dinh, Jascha Sohl-Dickstein, and Samy Bengio.
\newblock Density estimation using real nvp.
\newblock \emph{arXiv preprint arXiv:1605.08803}, 2016.

\bibitem[Du \& Mordatch(2019)Du and Mordatch]{du2019implicit}
Yilun Du and Igor Mordatch.
\newblock Implicit generation and generalization in energy-based models.
\newblock \emph{arXiv preprint arXiv:1903.08689}, 2019.

\bibitem[Garreau et~al.(2017)Garreau, Jitkrittum, and
  Kanagawa]{garreau2017large}
Damien Garreau, Wittawat Jitkrittum, and Motonobu Kanagawa.
\newblock Large sample analysis of the median heuristic.
\newblock \emph{arXiv preprint arXiv:1707.07269}, 2017.

\bibitem[Germain et~al.(2015)Germain, Gregor, Murray, and
  Larochelle]{germain2015made}
Mathieu Germain, Karol Gregor, Iain Murray, and Hugo Larochelle.
\newblock Made: Masked autoencoder for distribution estimation.
\newblock In \emph{International Conference on Machine Learning}, pp.\
  881--889, 2015.

\bibitem[Heusel et~al.(2017)Heusel, Ramsauer, Unterthiner, Nessler, and
  Hochreiter]{heusel2017gans}
Martin Heusel, Hubert Ramsauer, Thomas Unterthiner, Bernhard Nessler, and Sepp
  Hochreiter.
\newblock Gans trained by a two time-scale update rule converge to a local nash
  equilibrium.
\newblock \emph{arXiv preprint arXiv:1706.08500}, 2017.

\bibitem[Ho et~al.(2019)Ho, Chen, Srinivas, Duan, and Abbeel]{ho2019flow++}
Jonathan Ho, Xi~Chen, Aravind Srinivas, Yan Duan, and Pieter Abbeel.
\newblock Flow++: Improving flow-based generative models with variational
  dequantization and architecture design.
\newblock \emph{arXiv preprint arXiv:1902.00275}, 2019.

\bibitem[Ho et~al.(2020)Ho, Jain, and Abbeel]{ho2020denoising}
Jonathan Ho, Ajay Jain, and Pieter Abbeel.
\newblock Denoising diffusion probabilistic models.
\newblock \emph{arXiv preprint arXiv:2006.11239}, 2020.

\bibitem[Hyv{\"a}rinen(2005)]{hyvarinen2005estimation}
Aapo Hyv{\"a}rinen.
\newblock Estimation of non-normalized statistical models by score matching.
\newblock \emph{Journal of Machine Learning Research}, 6\penalty0
  (Apr):\penalty0 695--709, 2005.

\bibitem[Hyv{\"a}rinen(2008)]{hyvarinen2008optimal}
Aapo Hyv{\"a}rinen.
\newblock Optimal approximation of signal priors.
\newblock \emph{Neural Computation}, 20\penalty0 (12):\penalty0 3087--3110,
  2008.

\bibitem[K{\"a}{\"a}ri{\"a}inen(2006)]{kaariainen2006lower}
Matti K{\"a}{\"a}ri{\"a}inen.
\newblock Lower bounds for reductions.
\newblock In \emph{Atomic Learning Workshop}, 2006.

\bibitem[Kingma \& Dhariwal(2018)Kingma and Dhariwal]{kingma2018glow}
Durk~P Kingma and Prafulla Dhariwal.
\newblock Glow: Generative flow with invertible 1x1 convolutions.
\newblock In \emph{Advances in neural information processing systems}, pp.\
  10215--10224, 2018.

\bibitem[Krizhevsky et~al.(2009)]{krizhevsky2009learning}
Alex Krizhevsky et~al.
\newblock Learning multiple layers of features from tiny images.
\newblock 2009.

\bibitem[Lamb et~al.(2016)Lamb, Goyal, Zhang, Zhang, Courville, and
  Bengio]{lamb2016professor}
Alex~M Lamb, Anirudh Goyal Alias~Parth Goyal, Ying Zhang, Saizheng Zhang,
  Aaron~C Courville, and Yoshua Bengio.
\newblock Professor forcing: A new algorithm for training recurrent networks.
\newblock In \emph{Advances in neural information processing systems}, pp.\
  4601--4609, 2016.

\bibitem[Larochelle \& Murray(2011)Larochelle and Murray]{larochelle2011neural}
Hugo Larochelle and Iain Murray.
\newblock The neural autoregressive distribution estimator.
\newblock In \emph{Proceedings of the Fourteenth International Conference on
  Artificial Intelligence and Statistics}, pp.\  29--37. JMLR Workshop and
  Conference Proceedings, 2011.

\bibitem[Liu et~al.(2015)Liu, Luo, Wang, and Tang]{liu2015deep}
Ziwei Liu, Ping Luo, Xiaogang Wang, and Xiaoou Tang.
\newblock Deep learning face attributes in the wild.
\newblock In \emph{Proceedings of the IEEE international conference on computer
  vision}, pp.\  3730--3738, 2015.

\bibitem[Menick \& Kalchbrenner(2018)Menick and
  Kalchbrenner]{menick2018generating}
Jacob Menick and Nal Kalchbrenner.
\newblock Generating high fidelity images with subscale pixel networks and
  multidimensional upscaling.
\newblock \emph{arXiv preprint arXiv:1812.01608}, 2018.

\bibitem[Minnen et~al.(2018)Minnen, Ball{\'e}, and Toderici]{minnen2018joint}
David Minnen, Johannes Ball{\'e}, and George~D Toderici.
\newblock Joint autoregressive and hierarchical priors for learned image
  compression.
\newblock In \emph{Advances in Neural Information Processing Systems}, pp.\
  10771--10780, 2018.

\bibitem[Narayanan \& Mitter(2010)Narayanan and Mitter]{narayanan2010sample}
Hariharan Narayanan and Sanjoy Mitter.
\newblock Sample complexity of testing the manifold hypothesis.
\newblock In \emph{Advances in neural information processing systems}, pp.\
  1786--1794, 2010.

\bibitem[Oord et~al.(2016{\natexlab{a}})Oord, Dieleman, Zen, Simonyan, Vinyals,
  Graves, Kalchbrenner, Senior, and Kavukcuoglu]{oord2016wavenet}
Aaron van~den Oord, Sander Dieleman, Heiga Zen, Karen Simonyan, Oriol Vinyals,
  Alex Graves, Nal Kalchbrenner, Andrew Senior, and Koray Kavukcuoglu.
\newblock Wavenet: A generative model for raw audio.
\newblock \emph{arXiv preprint arXiv:1609.03499}, 2016{\natexlab{a}}.

\bibitem[Oord et~al.(2016{\natexlab{b}})Oord, Kalchbrenner, and
  Kavukcuoglu]{oord2016pixel}
Aaron van~den Oord, Nal Kalchbrenner, and Koray Kavukcuoglu.
\newblock Pixel recurrent neural networks.
\newblock \emph{arXiv preprint arXiv:1601.06759}, 2016{\natexlab{b}}.

\bibitem[Ostrovski et~al.(2018)Ostrovski, Dabney, and
  Munos]{ostrovski2018autoregressive}
Georg Ostrovski, Will Dabney, and R{\'e}mi Munos.
\newblock Autoregressive quantile networks for generative modeling.
\newblock \emph{arXiv preprint arXiv:1806.05575}, 2018.

\bibitem[Raphan \& Simoncelli(2011)Raphan and Simoncelli]{raphan2011least}
Martin Raphan and Eero~P Simoncelli.
\newblock Least squares estimation without priors or supervision.
\newblock \emph{Neural computation}, 23\penalty0 (2):\penalty0 374--420, 2011.

\bibitem[Rezende \& Mohamed(2015)Rezende and Mohamed]{rezende2015variational}
Danilo~Jimenez Rezende and Shakir Mohamed.
\newblock Variational inference with normalizing flows.
\newblock \emph{arXiv preprint arXiv:1505.05770}, 2015.

\bibitem[Roweis \& Saul(2000)Roweis and Saul]{roweis2000nonlinear}
Sam~T Roweis and Lawrence~K Saul.
\newblock Nonlinear dimensionality reduction by locally linear embedding.
\newblock \emph{science}, 290\penalty0 (5500):\penalty0 2323--2326, 2000.

\bibitem[Salimans et~al.(2016)Salimans, Goodfellow, Zaremba, Cheung, Radford,
  and Chen]{salimans2016improved}
Tim Salimans, Ian Goodfellow, Wojciech Zaremba, Vicki Cheung, Alec Radford, and
  Xi~Chen.
\newblock Improved techniques for training gans.
\newblock \emph{arXiv preprint arXiv:1606.03498}, 2016.

\bibitem[Salimans et~al.(2017)Salimans, Karpathy, Chen, and
  Kingma]{salimans2017pixelcnn++}
Tim Salimans, Andrej Karpathy, Xi~Chen, and Diederik~P Kingma.
\newblock Pixelcnn++: Improving the pixelcnn with discretized logistic mixture
  likelihood and other modifications.
\newblock \emph{arXiv preprint arXiv:1701.05517}, 2017.

\bibitem[Saremi \& Hyvarinen(2019)Saremi and Hyvarinen]{saremi2019neural}
Saeed Saremi and Aapo Hyvarinen.
\newblock Neural empirical bayes.
\newblock \emph{Journal of Machine Learning Research}, 20:\penalty0 1--23,
  2019.

\bibitem[Sietsma \& Dow(1991)Sietsma and Dow]{sietsma1991creating}
Jocelyn Sietsma and Robert~JF Dow.
\newblock Creating artificial neural networks that generalize.
\newblock \emph{Neural networks}, 4\penalty0 (1):\penalty0 67--79, 1991.

\bibitem[Sohl-Dickstein et~al.(2015)Sohl-Dickstein, Weiss, Maheswaranathan, and
  Ganguli]{sohl2015deep}
Jascha Sohl-Dickstein, Eric~A Weiss, Niru Maheswaranathan, and Surya Ganguli.
\newblock Deep unsupervised learning using nonequilibrium thermodynamics.
\newblock \emph{arXiv preprint arXiv:1503.03585}, 2015.

\bibitem[Song \& Ermon(2019)Song and Ermon]{song2019generative}
Yang Song and Stefano Ermon.
\newblock Generative modeling by estimating gradients of the data distribution.
\newblock In \emph{Advances in Neural Information Processing Systems}, pp.\
  11918--11930, 2019.

\bibitem[Song \& Ermon(2020)Song and Ermon]{song2020improved}
Yang Song and Stefano Ermon.
\newblock Improved techniques for training score-based generative models.
\newblock \emph{arXiv preprint arXiv:2006.09011}, 2020.

\bibitem[Szegedy et~al.(2013)Szegedy, Zaremba, Sutskever, Bruna, Erhan,
  Goodfellow, and Fergus]{szegedy2013intriguing}
Christian Szegedy, Wojciech Zaremba, Ilya Sutskever, Joan Bruna, Dumitru Erhan,
  Ian Goodfellow, and Rob Fergus.
\newblock Intriguing properties of neural networks.
\newblock \emph{arXiv preprint arXiv:1312.6199}, 2013.

\bibitem[Tenenbaum et~al.(2000)Tenenbaum, De~Silva, and
  Langford]{tenenbaum2000global}
Joshua~B Tenenbaum, Vin De~Silva, and John~C Langford.
\newblock A global geometric framework for nonlinear dimensionality reduction.
\newblock \emph{science}, 290\penalty0 (5500):\penalty0 2319--2323, 2000.

\bibitem[Theis et~al.(2015)Theis, Oord, and Bethge]{theis2015note}
Lucas Theis, A{\"a}ron van~den Oord, and Matthias Bethge.
\newblock A note on the evaluation of generative models.
\newblock \emph{arXiv preprint arXiv:1511.01844}, 2015.

\bibitem[Uria et~al.(2013)Uria, Murray, and Larochelle]{uria2013rnade}
Benigno Uria, Iain Murray, and Hugo Larochelle.
\newblock Rnade: The real-valued neural autoregressive density-estimator.
\newblock In \emph{Advances in Neural Information Processing Systems}, pp.\
  2175--2183, 2013.

\bibitem[Vincent(2011)]{vincent2011connection}
Pascal Vincent.
\newblock A connection between score matching and denoising autoencoders.
\newblock \emph{Neural computation}, 23\penalty0 (7):\penalty0 1661--1674,
  2011.

\bibitem[Vincent et~al.(2008)Vincent, Larochelle, Bengio, and
  Manzagol]{vincent2008extracting}
Pascal Vincent, Hugo Larochelle, Yoshua Bengio, and Pierre-Antoine Manzagol.
\newblock Extracting and composing robust features with denoising autoencoders.
\newblock In \emph{Proceedings of the 25th international conference on Machine
  learning}, pp.\  1096--1103, 2008.

\bibitem[Vinyals et~al.(2019)Vinyals, Babuschkin, Czarnecki, Mathieu, Dudzik,
  Chung, Choi, Powell, Ewalds, Georgiev, et~al.]{vinyals2019grandmaster}
Oriol Vinyals, Igor Babuschkin, Wojciech~M Czarnecki, Micha{\"e}l Mathieu,
  Andrew Dudzik, Junyoung Chung, David~H Choi, Richard Powell, Timo Ewalds,
  Petko Georgiev, et~al.
\newblock Grandmaster level in starcraft ii using multi-agent reinforcement
  learning.
\newblock \emph{Nature}, 575\penalty0 (7782):\penalty0 350--354, 2019.

\end{thebibliography}
\bibliographystyle{iclr2021_conference}

\appendix
\newpage
\appendix
\section{Proofs}
\label{app:proofs}
\lipschitz*
\begin{proof}
First, we have that:
\begin{align}
    \abs{\nabla_\vx p(\vx)} = \abs{p(\vx) \nabla_\vx \log p(\vx)} \leq Lip(p)
\end{align}
and if we assume symmetry, i.e. $q(\vx | \tilde{\vx}) = q(\tilde{\vx} | \vx)$ then by integration by parts we have:
\begin{align}
    \nabla_{\tilde{\vx}} q(\tilde{\vx}) = \E_{p(\vx)}[\nabla_{\tilde{\vx}} q(\tilde{\vx} | \vx)] = \E_{p(\vx)}[\nabla_{\vx} q(\vx | \tilde{\vx})] = -\E_{p(\vx)}[q(\vx | \tilde{\vx}) \nabla_\vx \log p(\vx) ]
\end{align}
Therefore, %
\begin{align}
   Lip(q) & = \max_{\tilde{\vx}} \abs{-\E_{p(\vx)}[q(\vx | \tilde{\vx}) \nabla_\vx \log p(\vx)]}  = \max_{\tilde{\vx}}
    \abs{\sum_\vx q(\vx | \tilde{\vx}) p(\vx) \nabla_\vx \log p(\vx)} \\
     & \leq \max_{\tilde{\vx}} \abs{
    \sum_\vx q(\vx | \tilde{\vx}) Lip(p)} =  Lip(p) \max_{\tilde{\vx}} \abs{
    \sum_\vx q(\vx | \tilde{\vx})}  = Lip(p)
\end{align}
which proves the result.
\end{proof}

\begin{restatable*}[Formal]{proposition}{regular}
Given a $D$-dimensional data distribution $p_{\text{data}}(\bfx)$ and model distribution $p_\theta(\bfx)$, assume that the smoothing distribution $q(\tilde\bfx|\bfx)$ satisfies:
\begin{itemize}
\item $\log p_\theta$ is infinitely differentiable on the support of $p_\theta(\bfx)$
 \item $q(\tilde\bfx|\bfx)$ is symmetric (\ie  $q(\tilde \bfx|\bfx)=q(\bfx|\tilde \bfx)$)
 \item $q(\tilde\bfx|\bfx)$ is stationary (\ie  translation invariant) 
 \item $q(\tilde\bfx|\bfx)$ is bounded and fully supported on $\mathbb{R}^D$
 \item $q(\tilde\bfx|\bfx)$ is element-wise independent
 \item $\E_{q(\tilde\bfx|\bfx)}[(\tilde\bfx-\bfx)^2]$ is bounded, and $\E_{q(\tilde\bfx|\bfx)}[(\tilde x_i-x_i)^2]=\eta$ at each dimension $i$.
\end{itemize}
Denote $\bfe=\tilde\bfx-\bfx$, then
\begin{equation*}
\small
    \E_{p_{\text{data}}(\bfx)}\E_{ q(\tilde\bfx|\bfx)}[\log p_{\theta}(\tilde\bfx)] = \E_{p_{\text{data}}(\bfx)}\left[\log p_\theta(\bfx) + \frac{\eta}{2} \sum_{i} \frac{\partial^2 \log p_\theta}{\partial x_i^2}\right] + \int \int o(\bfe^2) \pdata(\bfx) p(\bfe) \diff \bfx \diff \bfe,
\end{equation*}
where $o(\bfe^2):\mathbb{R}^D\to\mathbb{R}$ is a function of $\bfe$ such that $\lim_{\bfe\to 0}\frac{o(\bfe^2)}{\bfe^2}=0$. 
Thus when $\int \int o(\bfe^2) \pdata(\bfx) p(\bfe) \diff \bfx \diff \bfe\to 0$, we have
\begin{equation}
    \E_{p_{\text{data}}(\bfx)}\E_{ q(\tilde\bfx|\bfx)}[\log p_{\theta}(\tilde\bfx)]\to \E_{p_{\text{data}}(\bfx)}\left[\log p_\theta(\bfx) + \frac{\eta}{2} \sum_{i} \frac{\partial^2 \log p_\theta}{\partial x_i^2}\right].
\end{equation}
\end{restatable*}

\begin{proof}
To see this, we first note that the new training objective for the smoothed data distribution is $\E_{p_{\text{data}}(\bfx)}\E_{ q(\tilde\bfx|\bfx)}[\log p_{\theta}(\tilde\bfx)]$. 
Let $\bfe=\tilde\bfx-\bfx$, because of the assumptions we have, the PDF function $q(\tilde\bfx | \bfx)$ can be reparameterized as $p(\bfe)$ which satisfies: $p$ is bounded and fully supported on $\mathbb{R}^D$; $p$ is element-wise independent and $\E_{p(\bfe)}[\eps_i^2]=\E_{p(\eps_i)}[\eps_i^2]=\eta$ at each dimension $i$ ($i=1,...,D$).
Then we have
\begin{equation}
\label{eq:regularization}
    \E_{p_{\text{data}}(\bfx)}\E_{ q(\tilde\bfx|\bfx)}[\log p_{\theta}(\tilde\bfx)]=
    \int \int \log p_\theta(\bfx + \bfe) p_{\text{data}}(\bfx) p(\bfe) \diff \bfx \diff \bfe,
\end{equation}
Using Taylor expansion, we have:
$$
\log p_\theta(\bfx + \bfe) = \log p_\theta(\bfx) + \sum_i \epsilon_i \frac{\partial \log p_\theta}{\partial x_i} + \frac12 \sum_{i,j} \epsilon_i \epsilon_j \frac{\partial^2 \log p_\theta}{\partial x_i \partial x_j} + o(\bfe^2).
$$
Since $\bfe$ is independent of $\bfx$ and 
$$
\int_{-\infty}^{\infty}\epsilon_i \diff \epsilon_i = 0, \quad \int_{-\infty}^{\infty} \int_{-\infty}^{\infty} \epsilon_i \epsilon_j \diff \epsilon_i \diff \epsilon_j = \delta_{i,j} \eta,
$$
where $\delta_{i,j}$ is the Kronecker delta function, the right hand side of
\Eqref{eq:regularization} becomes
\begin{align}
    & \ \E_{p_{\text{data}}(\bfx)}[\log p_\theta(\bfx)]+\int\int \left(\frac12 \sum_{i} \epsilon_i^2 \frac{\partial^2 \log p_\theta}{\partial x_i^2}
    + o(\bfe^2) \right) p_{\text{data}}(\bfx) p(\bfe) \diff \bfx \diff \bfe \\
    = & \ \E_{p_{\text{data}}(\bfx)}\left[\log p_\theta(\bfx) + \frac{\eta}{2} \sum_{i} \frac{\partial^2 \log p_\theta}{\partial x_i^2}\right]+\int \int  o(\bfe^2) p_{\text{data}}(\bfx) p(\bfe) \diff \bfx \diff \bfe.
\end{align}
When
\begin{equation}
    \int \int o(\bfe^2) p_{\text{data}}(\bfx) p(\bfe) \diff \bfx \diff \bfe\to 0,
\end{equation}
we have
\begin{align}
    \E_{p_{\text{data}}(\bfx)}\E_{ q(\tilde\bfx|\bfx)}[\log p_{\theta}(\tilde\bfx)]\to  \E_{p_{\text{data}}(\bfx)}\left[\log p_\theta(\bfx) + \frac{\eta}{2} \sum_{i} \frac{\partial^2 \log p_\theta}{\partial x_i^2}\right].
\end{align}
\end{proof}

\section{Denoising experiments}
\label{app:2-d}
\subsection{Analysis on 1-d denoising}
\label{app:1-d}
To provide more insights into denoising, we first study ``single-step denoising" (see \eqref{eq:single_step}) on a 1-d dataset. We choose the data distribution to be a two mixture of Gaussian distribution $0.5\mathcal{N}(-0.3, 0.1^2)+0.5\mathcal{N}(0.3, 0.1^2)$ and the smoothing distribution to be $q(\tilde x|x)=\mathcal{N}(\tilde x|x, 0.3^2)$ (see \figref{fig:app:single_step_denoising_data}). Since the convolution of two Gaussian distributions is also a Gaussian distribution, the smoothed data is a mixture of Gaussian distribution given by $0.5\mathcal{N}(-0.3, 0.1^2+0.3^2)+0.5\mathcal{N}(0.3, 0.1^2+0.3^2)$. The ground truth of $\nabla_{\tilde x}\log p(\tilde x)$ can then be calculated in closed form. Thus, given the smoothed data $\tilde x$, we can calculate the ground truth $\nabla_{\tilde x}\log p(\tilde x)$ in \eqref{eq:single_step} and obtain $\bar{x}$ using ``single-step denoising". We visualize the denoising results in \figref{fig:app:single_step_denoising_result}. We find that the low density region between the two modes in $p_{\text{data}}(x)$ are not modeled properly in \figref{fig:app:single_step_denoising_result}. However, this is very expected since ``single-step denosing" uses $\bar{x}=\mathbb{E}[x|\tilde x]$ as the substitute for the denoised result. When the smoothing distribution has a large variance (like in \figref{fig:app:single_step_denoising_data} where the smoothed data has merged into a one mode distribution), datapoints like $\tilde x_0$ in the middle low density region of $p_{\text{data}}(x)$ can have high density in the smoothed distribution. 
Since $\tilde x_0$, as well as other points in the middle low density region of $p_{\text{data}}(x)$, can come from both modes of $p_{\text{data}}(x)$ with high probability before the smoothing process (see \figref{fig:app:ground_truth_denoising_data_11}), the denoised $\bar{x}=\mathbb{E}[x|\tilde x=\tilde x_0]$ can still be located in the middle low density region (see \figref{fig:app:single_step_denoising_result}). Since a large proportion of the smoothed data is located in the middle low density region of $p_{\text{data}}(x)$, we would expect certain proportion of the density to remain in the low density region after ``single-step denoising" just as shown in \figref{fig:app:single_step_denoising_result}. However, when the smoothing distribution has a smaller variance, ``single-step denoising" can achieve much better denoising results (see \figref{fig:app:single_step_denoising_0.1}, where we use $q(\tilde x|x)=\mathcal{N}(\tilde x|x, 0.1^2)$). 
Although denoising can be easier when the smoothing distribution has a smaller variance, modeling the smoothed distribution could be harder as we discussed before.

In general, the right denoising results should be samples coming from $p(x|\tilde x)$, which is the reason why samples from $p_{\theta}(x|\tilde x)$ (\ie introducing the model $p_{\theta}(x|\tilde x)$) is more ideal than using $\mathbb{E}_{\theta}[x|\tilde x]$ as a denoising substitute (\ie ``single-step denoising").
In general, the capacity of the denoising model $p_{\theta}(x|\tilde x)$ also matters in terms of denoising results. Let us again consider the datapoint $\tilde x_0$ shown in \figref{fig:app:single_step_denoising_data}. If the invert smoothing model $p_{\theta}(x|\tilde x)$ is a one mode logistic distribution, due to the mode covering property of maximum likelihood estimation, given the smoothed observation $\tilde x_0$, the best the model can do is to center its only mode at $\tilde x_0$ for approximating $p(x|\tilde x=\tilde x_0)$ (see \figref{fig:app:1model_denoising_x0}). Thus, like $\tilde x_0$, the smoothed datapoints at the low density region between the two modes of $p_{\text{data}}(x)$ are still likely to remain between the two modes after denoising
(see \figref{fig:app:1model_denoising}). To solve this issue, we can increase the capacity of $p_{\theta}(x|\tilde x)$ by making it a two mixture of logistics. 
In this case, the distribution $p_{\theta}(x|\tilde x=\tilde x_0)$ can be captured in a better way (see \figref{fig:app:2model_denoising_x0} and \figref{fig:app:ground_truth_denoising_data_11}). 
After the invert smoothing process, like $\tilde x_0$, most smoothed datapoints in the low density can be mapped to one of the two high density modes (see \figref{fig:app:2model_denoising}), resulting in much better denoising effects.

\begin{figure}[H]
    \centering
    \begin{subfigure}{0.45\textwidth}
    \includegraphics[width=\textwidth]{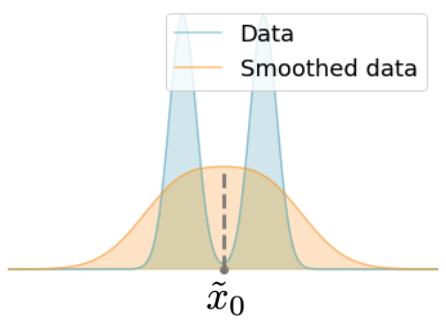}
    \caption{Data distribution ($q(\tilde x|x)=\mathcal{N}(\tilde{x}|x, 0.3^2)$).}
    \label{fig:app:single_step_denoising_data}
    \end{subfigure}
    \begin{subfigure}{0.45\textwidth}
    \vspace*{-0.1in}
    \includegraphics[width=\textwidth]{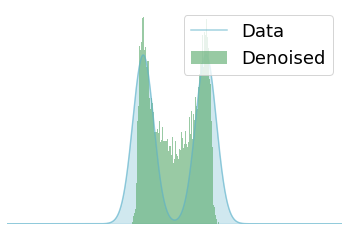}
    \vspace*{0.1in}
    \caption{Single-step denoising results.}
    \label{fig:app:single_step_denoising_result}
    \end{subfigure}
    \caption{1-d single-step (gradient based) denoising.
   }
    \label{fig:app:single_step_denoising}
\end{figure}

\begin{figure}[H]
    \centering
    \begin{subfigure}{0.45\textwidth}
    \includegraphics[width=\textwidth]{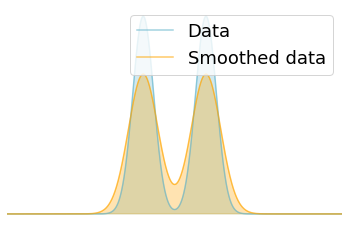}
    \caption{Data distribution ($q(\tilde x|x)=\mathcal{N}(\tilde{x}|x, 0.1^2)$).}
    \label{fig:app:single_step_denoising_data_0.1}
    \end{subfigure}
    \begin{subfigure}{0.5\textwidth}
    \vspace*{-0.2in}
    \includegraphics[width=\textwidth]{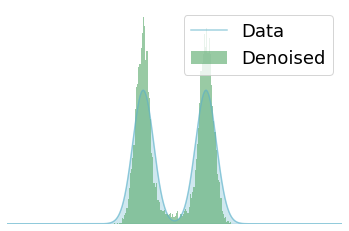}
    \vspace*{-0.1in}
    \caption{Single-step denoising results.}
    \label{fig:app:single_step_denoising_result_0.1}
    \end{subfigure}
    \caption{1-d single-step (gradient based) denoising.
   }
    \label{fig:app:single_step_denoising_0.1}
\end{figure}

\begin{figure}[H]
    \centering
    \begin{subfigure}{0.32\textwidth}
    \includegraphics[width=\textwidth]{figures/single_step_denoise/single_step_denoise_data.png}
    \caption{Data distribution.
    }
    \end{subfigure}
    \begin{subfigure}{0.32\textwidth}
    \vspace*{-0.1in}
    \includegraphics[width=\textwidth]{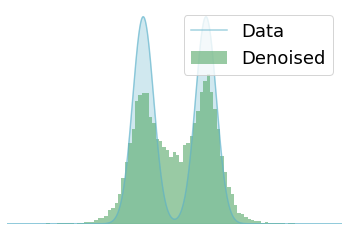}
    \vspace*{0.01in}
    \caption{Denoised with $p_{\theta}(x|\tilde x)$.}
    \label{fig:app:1model_denoising}
    \end{subfigure}
    \begin{subfigure}{0.32\textwidth}
    \includegraphics[width=\textwidth]{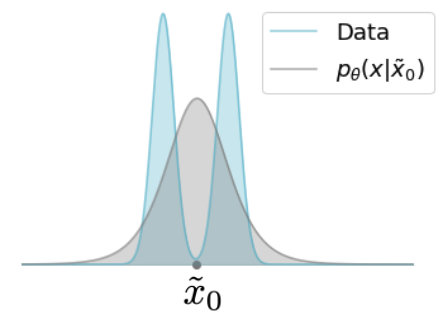}
    \caption{Distribution of $p_{\theta}(x|\tilde x=\tilde x_0)$.}
    \label{fig:app:1model_denoising_x0}
    \end{subfigure}
    \caption{Denoising with $p_{\theta}(x|\tilde x)$, which is modeled by one mixture of logistics.
   }
    \label{fig:app:1model}
\end{figure}

\begin{figure}[H]
    \centering
     \begin{subfigure}{0.32\textwidth}
    \includegraphics[width=\textwidth]{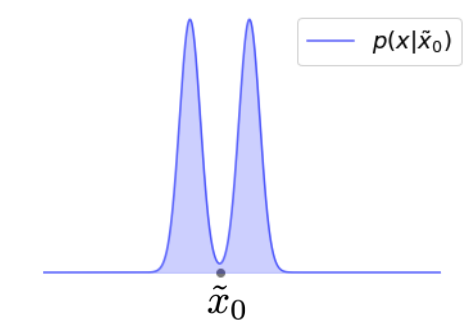}
    \caption{Ground truth $p(x|\tilde x=\tilde x_0)$.}
    \label{fig:app:ground_truth_denoising_data_11}
    \end{subfigure}
    \begin{subfigure}{0.32\textwidth}
    \vspace*{-0.1in}
    \includegraphics[width=\textwidth]{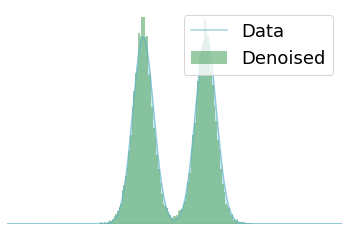}
    \vspace*{0.01in}
    \caption{Denoising with $p_{\theta}(x|\tilde x)$.}
    \label{fig:app:2model_denoising}
    \end{subfigure}
    \begin{subfigure}{0.32\textwidth}
    \includegraphics[width=\textwidth]{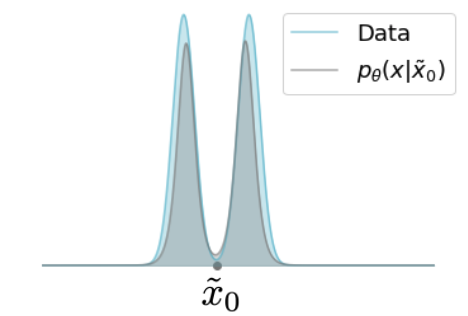}
    \caption{Distribution of $p_{\theta}(x|\tilde x=\tilde x_0)$.}
    \label{fig:app:2model_denoising_x0}
    \end{subfigure}
    \caption{Denoising with $p_{\theta}(x|\tilde x)$, which is modeled by two mixtures of logistics.
   }
    \label{fig:app:2model}
\end{figure}

\subsection{Analysis on 2-d denoising}
\label{app:2-d_denoise}
On the 2-d Olympics dataset in section \ref{sec:2-d-dataset}, we find that the intersections between rings can be poorly modeled with the proposed smoothing approach when only two mixture of logistics are used (see \figref{fig:app:2d_noisy_samples_2}). 
We believe this can be caused if the denoising model is not flexible enough to capture the distribution $p(\bfx|\tilde \bfx)$.
More specifically, we note that
the ground truth distribution for $p(\bfx|\tilde \bfx)$ at the intersections of the rings is a highly complicated distribution and can be hard to capture using our model which only has two mixtures of logistics for each dimension. 
If we increase the flexibility of $p_{\theta}(\bfx|\tilde \bfx)$ by using three or four mixtures of logistics components (note that we still use fewer mixture components than the MADE baseline and we use comparable number of parameters), 
the intersection of the rings can be modeled in an improved way (see \figref{fig:app:2d_example}).

We also provide ``single-step denoising" results for the experiments in Section~\ref{sec:2-d-dataset} (see \figref{fig:app:single_step}), where we use the same smoothing distribution, and the MADE model with three mixture components as used in section \ref{sec:2-d-dataset}. We note that ``single-step denoising" results are not very good, which is also expected. As discussed in section \ref{app:1-d},
when the smoothing distribution has a relatively large variance, $\mathbb{E}_{\theta}[\bfx|\tilde\bfx]$ is not a good approximation for the denoised result, and we want the denoised sample to come from the distribution $p_{\theta}(\bfx|\tilde \bfx)$, in which case introducing a denoising model $p_{\theta}(\bfx|\tilde \bfx)$ could be a better option. Although we could select $q(\tilde\bfx|\bfx)$ to have a smaller variance so that ``single-step denoing" could work reasonably well, but modeling $p(\tilde \bfx)$ in this case could be more challenging.

\begin{figure}
    \centering
    \begin{subfigure}{0.19\textwidth}
    \includegraphics[width=\textwidth]{figures/2d_visualization/olympics_data.png}
    \caption{Data}
    \label{fig:app:2d_data_2}
    \end{subfigure}
    \begin{subfigure}{0.19\textwidth}
    \includegraphics[width=\textwidth]{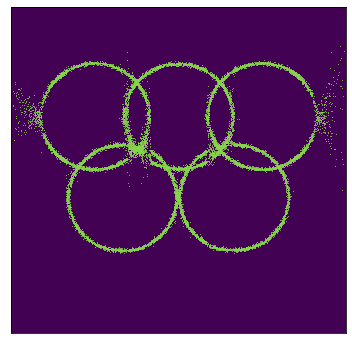}
    \caption{MADE ($n=5$)}
    \label{fig:app:2d_baseline_2}
    \end{subfigure}
    \begin{subfigure}{0.19\textwidth}
    \includegraphics[width=\textwidth]{figures/2d_visualization/made_mle_6.png}
    \caption{MADE ($n=6$)}
    \label{fig:app:2d_baseline_2}
    \end{subfigure}
    \begin{subfigure}{0.19\textwidth}
    \includegraphics[width=\textwidth]{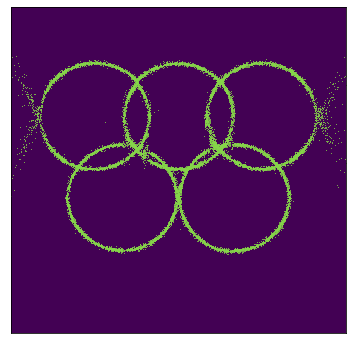}
    \caption{MADE ($n=7$)}
    \label{fig:app:2d_baseline_2}
    \end{subfigure}
    \\
    \begin{subfigure}{0.19\textwidth}
    \includegraphics[width=\textwidth]{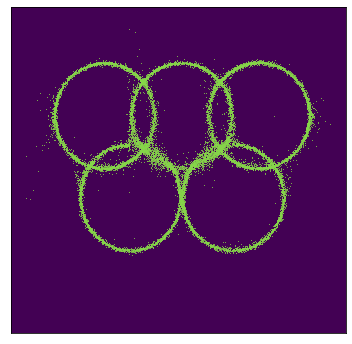}
    \caption{Ours ($n=2$)}
    \label{fig:app:2d_noisy_samples_2}
    \end{subfigure}
    \begin{subfigure}{0.19\textwidth}
    \includegraphics[width=\textwidth]{figures/2d_visualization/noise_conditioned_olympics.png_3_mixture.png}
    \caption{Ours ($n=3$)}
    \label{fig:app:2d_noisy_samples_3}
    \end{subfigure}
    \begin{subfigure}{0.19\textwidth}
    \includegraphics[width=\textwidth]{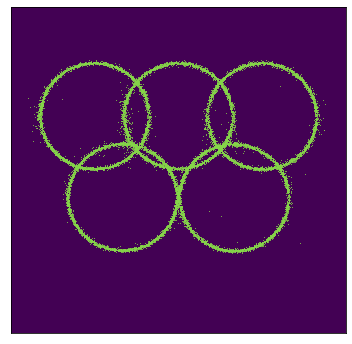}
    \caption{Ours ($n=4$)}
    \label{fig:app:2d_noisy_samples_3}
    \end{subfigure}
    \caption{Samples on 2-d synthetic datasets. We use a MADE model with comparable number of parameters for both our method and the baseline. The models have $n$ mixture of logistics for each dimension. 
    Our method is able to obtain reasonable samples when using fewer mixture components, while the baseline still has trouble modeling the two sides of the rings when $n=7$.}
    \label{fig:app:2d_example}
\end{figure}

\begin{figure}
    \centering
    \begin{subfigure}{0.3\textwidth}
    \includegraphics[width=\textwidth]{figures/2d_visualization/toy_data.png}
    \caption{Data}
    \end{subfigure}
    \begin{subfigure}{0.3\textwidth}
    \includegraphics[width=\textwidth]{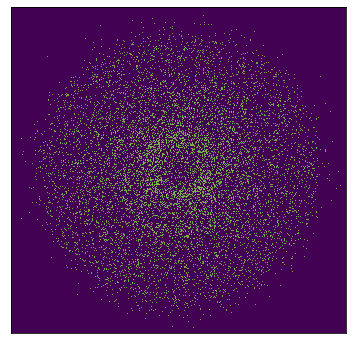}
    \caption{Smoothed data distribution}
    \end{subfigure}
    \begin{subfigure}{0.3\textwidth}
    \includegraphics[width=\textwidth]{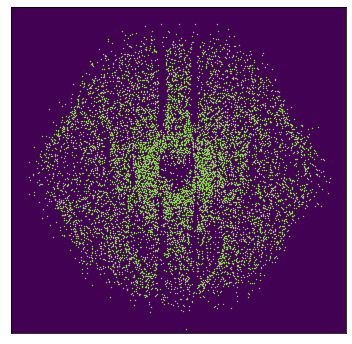}
    \caption{Single-step denoising results}
    \end{subfigure}
    \vfill
    \begin{subfigure}{0.3\textwidth}
    \includegraphics[width=\textwidth]{figures/2d_visualization/olympics_data.png}
    \caption{Data}
    \end{subfigure}
    \begin{subfigure}{0.3\textwidth}
    \includegraphics[width=\textwidth]{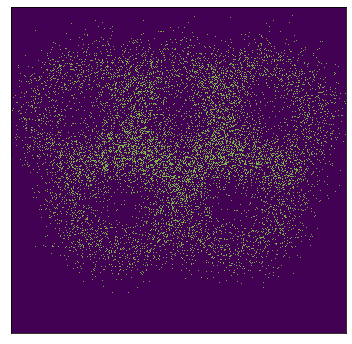}
    \caption{Smoothed data distribution}
    \end{subfigure}
    \begin{subfigure}{0.3\textwidth}
    \includegraphics[width=\textwidth]{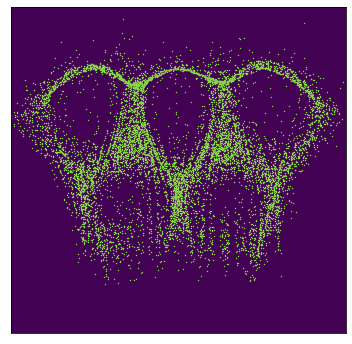}
    \caption{Single-step denoising results}
    \end{subfigure}
    \caption{Single-step denoising results on 2-d synthetic datasets. We use the same MADE model with three mixture components and the same smoothing distribution as mentioned in Section \ref{sec:2-d-dataset}.}
    \label{fig:app:single_step}
\end{figure}

\section{Image experiments}
\label{app:image}
\subsection{Settings}
For the image experiments, we first rescale images to $[-1,1]$ and then perturb the images with $q(\tilde\bfx|\bfx)=\mathcal{N}(\tilde\bfx|\bfx,\sigma^2I)$. We use $\sigma=0.5$ for MNIST and $\sigma=0.3$ for both CIFAR-10 and CelebA. The selection of $\sigma$ is mainly based on analysis in \citep{saremi2019neural}. More specifically, given an image, we consider the median value of the Euclidean distance between two data points in a dataset, and then divide it by $2\sqrt{D}$, where $D$ is the dimension of the data. This provides us with a way of selecting the variance of $q(\tilde\bfx|\bfx)$, when $q(\tilde\bfx|\bfx)$ is a Gaussian distribution. We find this selection of variance able to generate reasonably well samples in practice. 
We train all the models with Adam optimizer with learning rate $0.0002$. To model $p_{\theta}(\bfx|\tilde\bfx)$, we stack $\tilde\bfx$ and $\bfx$ together at the second dimension to obtain $\hat\bfx=[\tilde\bfx, \bfx]$, which ensures that $\tilde\bfx$ comes before $\bfx$ in the pixel ordering. For instance, this stacking would provide an image $\hat\bfx$ with size $1\times (2\times 28)\times 28$ on a MNIST image, and an image with size $3\times (2\times 32)\times 32$ on a CIFAR-10 image.
Since PixelCNN++ consists of convolutional layers, we can directly feed $\hat\bfx$ into the default architecture without modifying the model architecture. As the latter pixels of the input only depend on the previous pixels in an autoregressive model and $\tilde\bfx$ comes before $\bfx$, we can parameterize $p_{\theta}(\bfx|\tilde\bfx)$ by computing the likelihoods only on $\bfx$ using the outputs from the autoregressive model.

\subsection{Image Inpainting}
\label{app:image_inpainting}
Since both $p_{\theta}(\tilde \bfx)$ and $p_{\theta}(\bfx|\tilde \bfx)$ are parameterized by an autoregressive model, we can also perform image inpainting using our method. We present the inpainting results on CIFAR-10 in \figref{fig:cifar10_inpaint} and CelebA in \figref{fig:celeba_inpaint}, where the bottom half of the input image is being inpainted.
\begin{figure}[H]
    \begin{subfigure}{0.49\linewidth}
    \includegraphics[width=\textwidth, trim=0 0 0 0, clip]{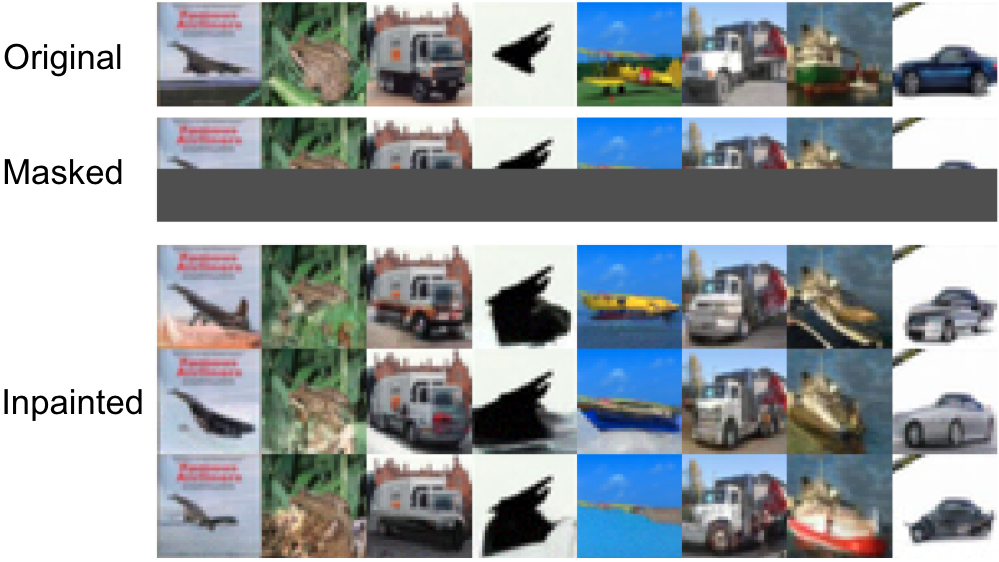}
    \caption{CIFAR10 inpainting}
    \label{fig:cifar10_inpaint}
    \end{subfigure}
    \hfill
    \begin{subfigure}{0.49\linewidth}
    \includegraphics[width=\textwidth, trim=0 0 0 0, clip]{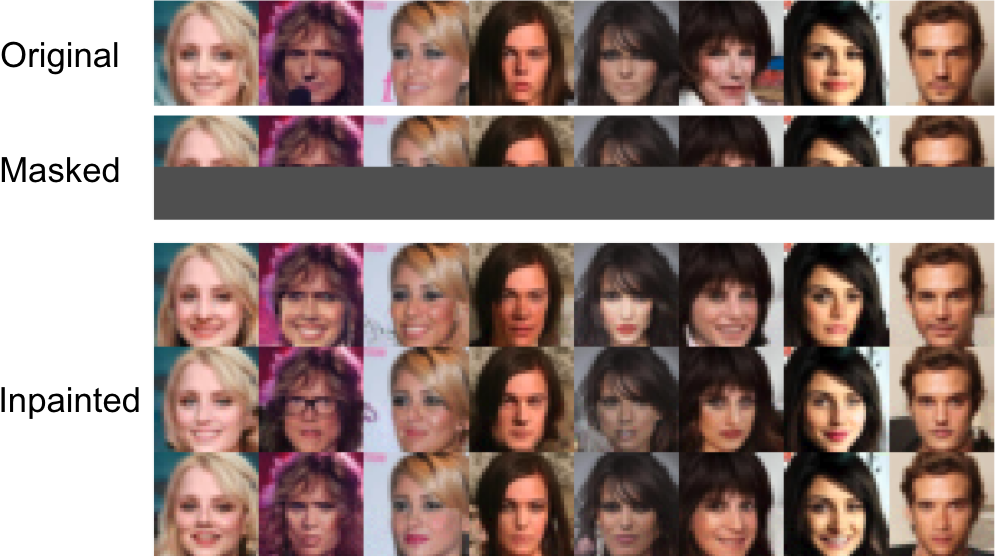}
    \caption{CelebA inpainting}
    \label{fig:celeba_inpaint}
    \end{subfigure}
    \caption{Inpainting results from our two-step method. The bottom half of the images are masked for inpainting.} 
\end{figure}

\subsection{Image Denoising}
\label{app:image_denoising}
We notice that the reverse smoothing process can also be understood as a denoising process. 
Besides the ``single-step denoising" approach shown above, we can also apply $p_{\theta}(\bfx|\tilde\bfx)$ to denoise images.
To visualize the denoising performance, we sample $\bfx_{\text{test}}$ from the test set and perturb $\bfx_{\text{test}}$ with $q(\tilde \bfx|\bfx)$ to obtain a noisy sample $\tilde \bfx _{\text{test}}$. We feed $\tilde \bfx_{\text{test}}$ into  $p_{\theta}(\bfx|\tilde \bfx=\tilde \bfx_{\text{test}})$ and draw samples from the model. We visualize the results in \figref{fig:denoising}. 
As we can see, the model exhibits reasonable denoising results, which shows that the autoregressive model is capable of learning the data distribution when conditioned on the smoothed data. 

\begin{figure}[H]
\vspace{-3pt}
    \centering
    \begin{subfigure}{0.22\textwidth}
    \adjincludegraphics[width=\textwidth,trim={0 {0.5\height} 0 0},clip]{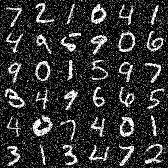}
    \caption{Noisy MNIST}
    \end{subfigure}
    ~
    \begin{subfigure}{0.22\textwidth}
    \adjincludegraphics[width=\textwidth,trim={0 {0.5\height} 0 0},clip]{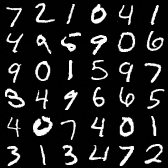}
    \caption{MNIST denoising}
    \end{subfigure}
    \begin{subfigure}{0.22\textwidth}
    \adjincludegraphics[width=\textwidth,trim={0 {0.5\height} 0 0},clip]{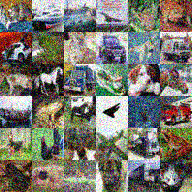}
    \caption{Noisy CIFAR10}
    \label{fig:mnist_noisy}
    \end{subfigure}
    ~
    \begin{subfigure}{0.22\textwidth}
    \adjincludegraphics[width=\textwidth,trim={0 {0.5\height} 0 0},clip]{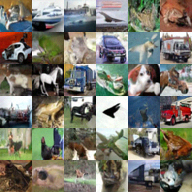}
    \caption{CIFAR10 denoising}
    \end{subfigure}
    \caption{Denoising with $p_{\theta}(\bfx|\tilde \bfx)$}
    \label{fig:denoising}
\end{figure}

\newpage
\subsection{More samples}
\label{app:image_samples}
\FloatBarrier
\vfill
\begin{figure}[H]
    \centering
    \includegraphics[width=0.9\textwidth]{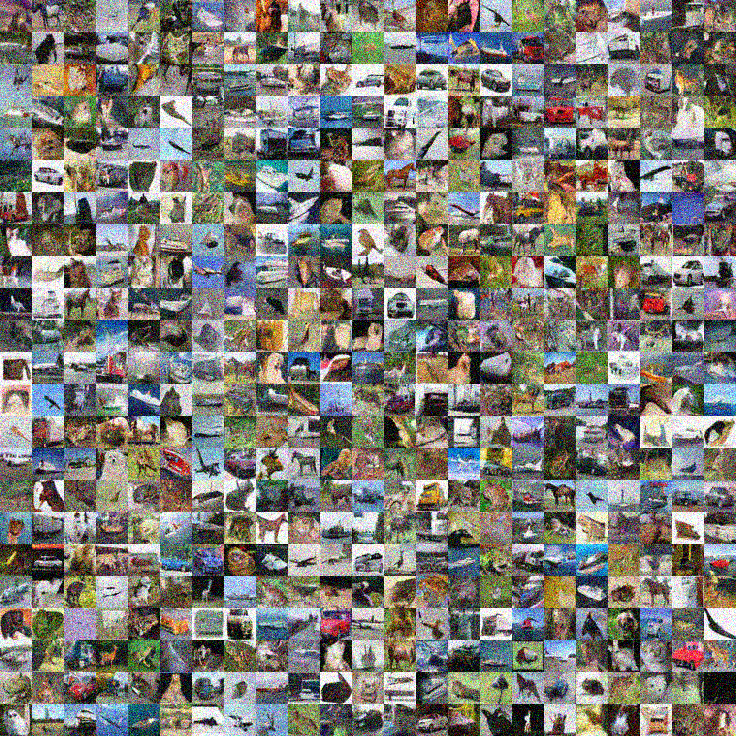}
    \caption{CIFAR-10 samples from $p_{\theta}(\tilde \bfx)$ (unconditioned on class labels).}
    \label{fig:cifar10_noisy_samples_app}
\end{figure}
\FloatBarrier
\vfill
\newpage
\vspace*{\fill}
\FloatBarrier
\begin{figure}[H]
    \centering
    \includegraphics[width=0.9\textwidth]{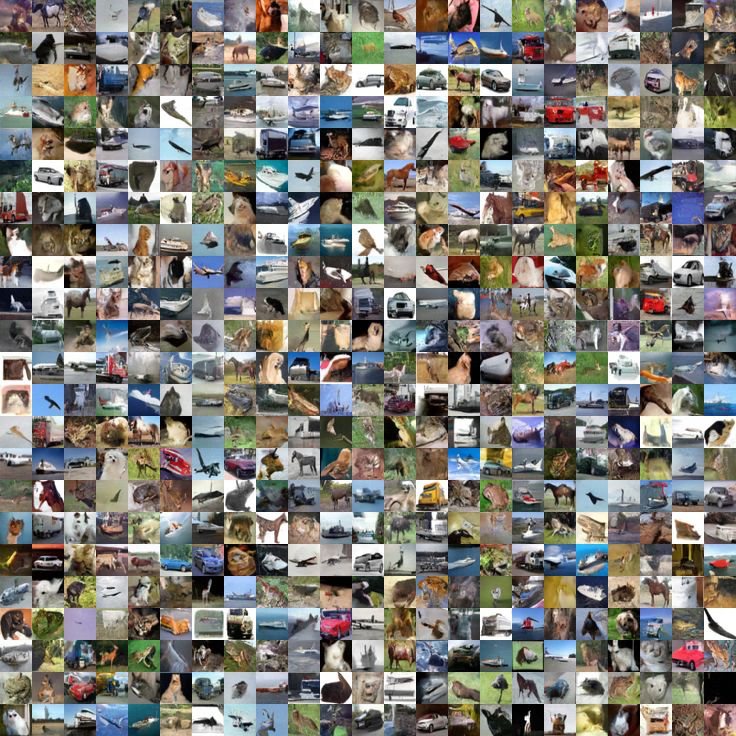}
    \caption{CIFAR-10 samples from $p_{\theta}(\bfx|\tilde \bfx)$ (unconditioned on class labels).}
    \label{fig:cifar10_samples_app}
\end{figure}
\FloatBarrier
\vfill
\newpage
\vspace*{\fill}
\FloatBarrier
\begin{figure}[H]
    \centering
    \includegraphics[width=0.9\textwidth]{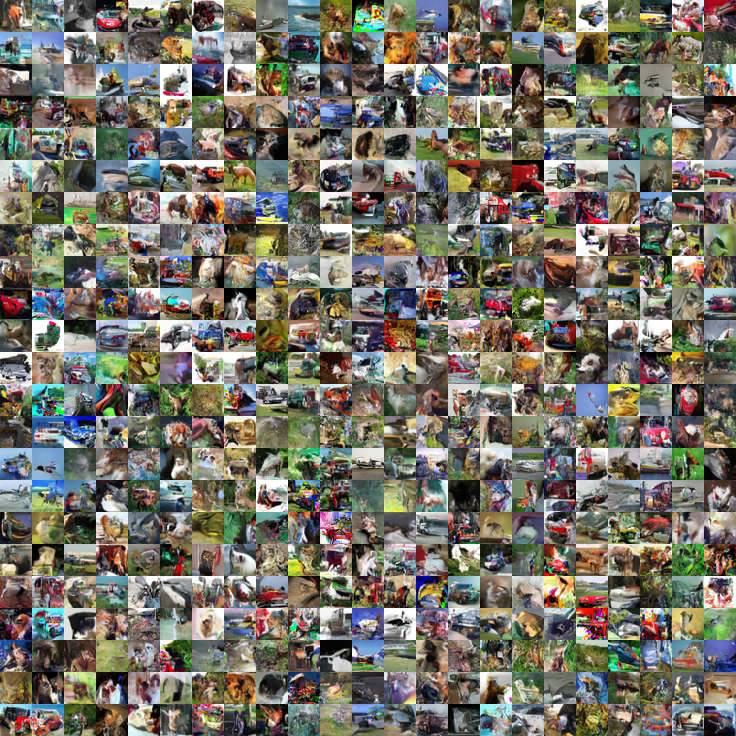}
    \caption{CIFAR-10 samples from the original PixelCNN++ method (unconditioned on class labels).}
    \label{fig:cifar10_original_samples_app}
\end{figure}
\FloatBarrier
\vfill
\newpage
\vspace*{\fill}
\FloatBarrier
\begin{figure}[H]
    \centering
    \includegraphics[width=0.9\textwidth]{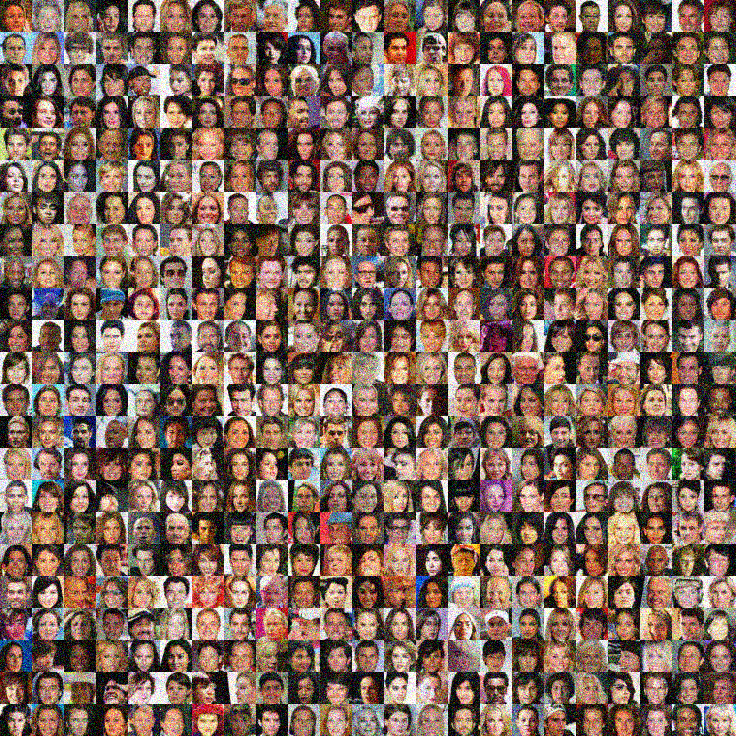}
    \caption{CelebA samples from $p_{\theta}(\tilde \bfx)$ (unconditioned on class labels).}
    \label{fig:celeba_noisy_samples_app}
\end{figure}
\FloatBarrier
\vfill
\newpage
\vspace*{\fill}
\FloatBarrier
\begin{figure}[H]
    \centering
    \includegraphics[width=0.9\textwidth]{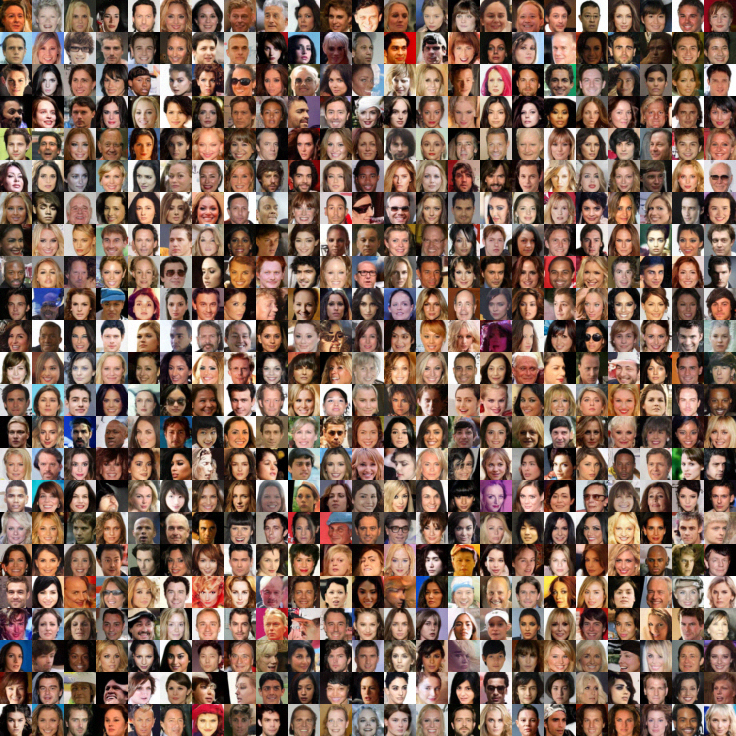}
    \caption{CelebA samples from $p_{\theta}(\bfx|\tilde \bfx)$ (unconditioned on class labels).}
    \label{fig:celeba_samples_app}
\end{figure}
\FloatBarrier
\vfill
\newpage
\vspace*{\fill}
\FloatBarrier
\begin{figure}[H]
    \centering
    \includegraphics[width=0.9\textwidth]{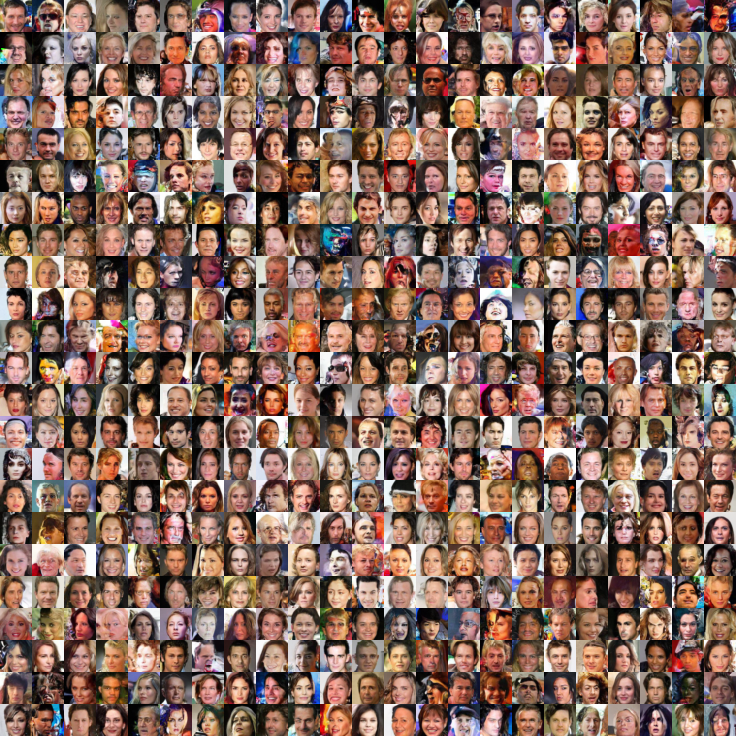}
    \caption{CelebA samples from the original PixelCNN++ method (unconditioned on class labels).}
    \label{fig:celeba_original_samples_app}
\end{figure}
\subsection{Nearest neighbors}
\vspace*{\fill}
\FloatBarrier
\label{app:neighbor}
\begin{figure}
    \centering
    \includegraphics[width=0.5\textwidth]{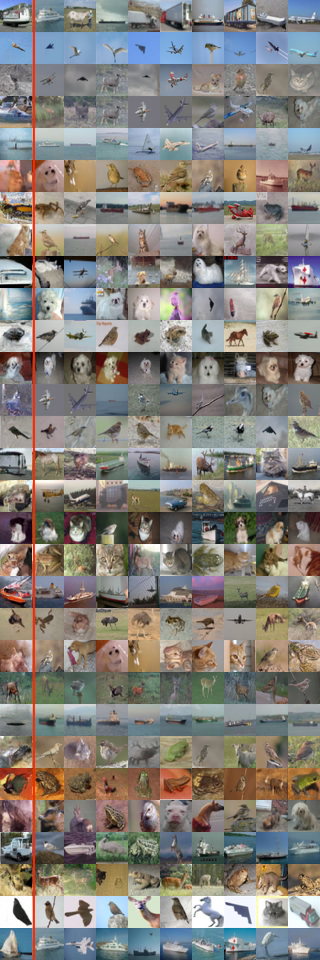}
    \caption{ Nearest neighbors measured by the $\ell_2$ distance between images. Images on the left of the red vertical line are samples from our model. Images on the right are nearest neighbors in the training dataset.}
    \label{fig:cifar10_l2_neighbor_inception}
\end{figure}

\begin{figure}
    \centering
    \includegraphics[width=0.5\textwidth]{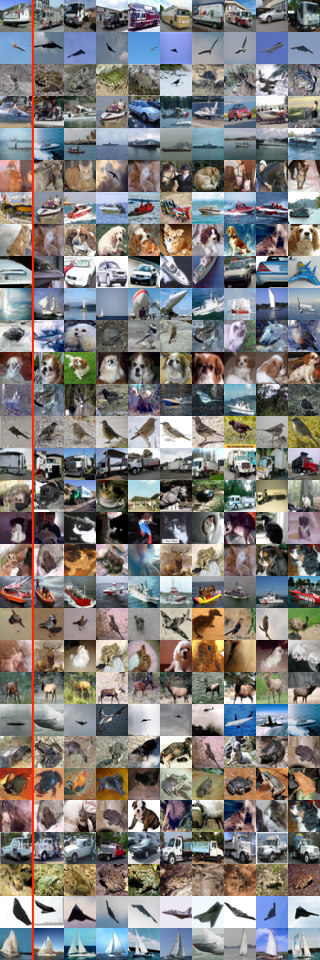}
    \caption{Nearest neighbors measured by the $\ell_2$ distance in the feature space of an Inception V3 network pretrained on ImageNet. Images on the left of the red vertical line are samples from our model. Images on the right are nearest neighbors in the training dataset. 
    }
    \label{fig:cifar10_neighbor_inception}
\end{figure}
\newpage
\begin{figure}[H]
    \centering
    \includegraphics[width=0.5\textwidth]{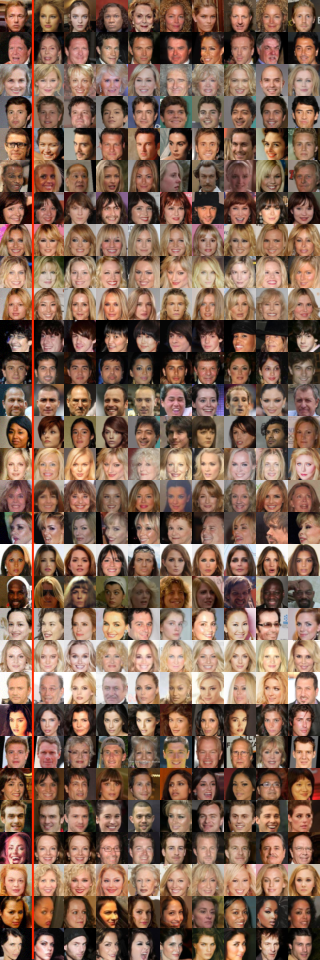}
    \caption{ Nearest neighbors measured by the $\ell_2$ distance between images. Images on the left of the red vertical line are samples from our model. Images on the right are nearest neighbors in the training dataset.}
    \label{fig:celeba_l2_neighbor_inception}
\end{figure}
\newpage
\begin{figure}[H]
    \centering
    \includegraphics[width=0.5\textwidth]{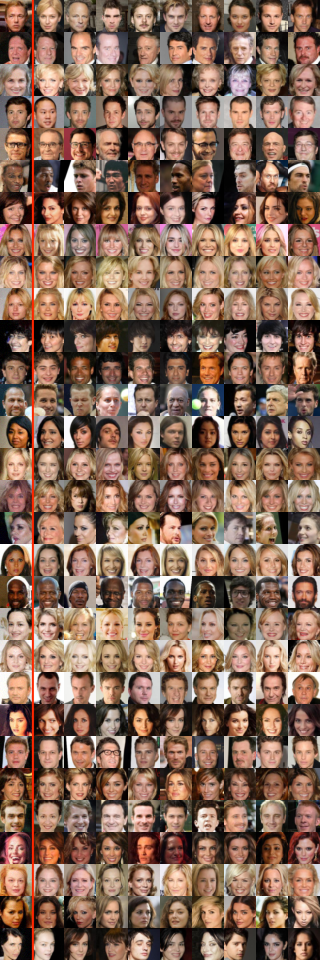}
    \caption{Nearest neighbors measured by the $\ell_2$ distance in the feature space of an Inception V3 network pretrained on ImageNet. Images on the left of the red vertical line are samples from our model. Images on the right are nearest neighbors in the training dataset.}
    \label{fig:celeba_neighbor_inception}
\end{figure}

\subsection{Ablation studies}
In this section, we show that gradient-based ``single-step denoising" will not improve sample qualities without performing randomized smoothing. To see this, we draw samples from a PixelCNN++ $p_{\theta}(\bfx)$ trained directly on $p_{\text{data}}(\bfx)$ (\ie without smoothing).
We perform ``single-step denoising" update defined as
\begin{equation}
    \bfx = \bfx + \sigma^2\nabla_{\bfx}\log p_{\theta}(\bfx).
\end{equation}
We explore various values for $\sigma$, and report the results in \figref{fig:app:sigma_unsmooth}. This shows that ``single-step denoising" alone (without randomized smoothing) will not improve sample quality of PixelCNN++.

\begin{figure}[H]
    \centering
    \begin{subfigure}{0.23\textwidth}
    \includegraphics[width=\textwidth]{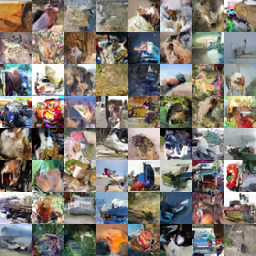}
    \caption{$\sigma=0$ (Original)}
    \end{subfigure}
    \begin{subfigure}{0.23\textwidth}
    \includegraphics[width=\textwidth]{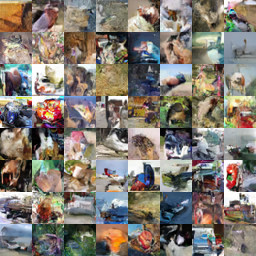}
    \caption{$\sigma=0.01$}
    \end{subfigure}
    \begin{subfigure}{0.23\textwidth}
    \includegraphics[width=\textwidth]{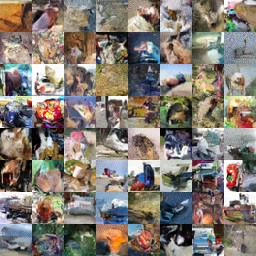}
    \caption{$\sigma=0.05$}
    \end{subfigure}
    \begin{subfigure}{0.23\textwidth}
    \includegraphics[width=\textwidth]{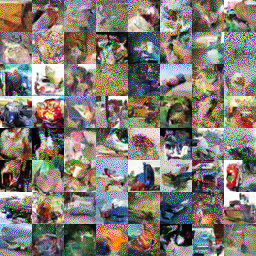}
    \caption{$\sigma=0.1$}
    \end{subfigure}
    \caption{``Single-step denoising" on PixelCNN++ trained on un-smoothed data. $\sigma=0$ corresponds to the original samples.
    }
    \label{fig:app:sigma_unsmooth}
\end{figure}

\end{document}